\documentclass[final,onefignum,onetabnum]{siamart171218}

\usepackage[a4paper, margin=1.2in]{geometry}
\usepackage{mathrsfs}
\usepackage{microtype}
\usepackage{csvsimple}

\usepackage{lipsum}
\usepackage{amsfonts}
\usepackage{graphicx}
\usepackage{epstopdf}
\usepackage{dsfont}
\usepackage{amssymb}
\usepackage{enumitem}
\usepackage[sort, square, numbers]{natbib}

\usepackage{algorithm} %
\usepackage{algpseudocode} %

\ifpdf
  \DeclareGraphicsExtensions{.eps,.pdf,.png,.jpg}
\else
  \DeclareGraphicsExtensions{.eps}
\fi

\newsiamremark{remark}{Remark}
\newsiamremark{hypothesis}{Hypothesis}
\crefname{hypothesis}{Hypothesis}{Hypotheses}
\newsiamthm{claim}{Claim}
\newsiamremark{problem}{Problem}

\newsiamremark{example}{Example}

\headers{Spatial Transcriptomics via adaptive regularization}{S Kubal, N Graham, M Heitz, A Warren, MP FriedLander, Y Plan, G Schiebinger}

\title{STARK denoises spatial transcriptomics images via \\
adaptive regularization}

\author{Sharvaj~Kubal\footnotemark[1] \footnotemark[5]
\and Naomi~Graham\footnotemark[2] \footnotemark[3] \footnotemark[5]
\and Matthieu~Heitz\footnotemark[1]
\and Andrew~Warren\footnotemark[1]
\and Michael~P.~Friedlander\footnotemark[1] \footnotemark[2]
\and Yaniv~Plan\footnotemark[1]
\and Geoffrey~Schiebinger\footnotemark[1] \footnotemark[4]
}

\usepackage{amsopn}

\newcommand{\lsc}{l}

\makeatletter
\newcommand*{\addFileDependency}[1]{%
  \typeout{(#1)}%
  \@addtofilelist{#1}%
  \IfFileExists{#1}{}{\typeout{No file #1.}}%
}
\makeatother

\newcommand{\bR}{\mathbb{R}}

\newcommand{\fro}[1]{\|#1\|_{F}}
\newcommand{\norm}[1]{\|#1\|}

\newcommand{\euc}[1]{\left| #1 \right|}

\newcommand{\op}[1]{\|#1\|_{\mathrm{op}}}

\newcommand{\argmin}{\mathop{\arg \min}}

\newcommand{\Jinf}{\mathbf{J}_\infty}

\newcommand{\indicator}{\mathds{1}}

\newcommand{\minimize}{\mathop{\mathrm{minimize}}}

\newcommand{\tr}{\mathop{\mathrm{tr}}}

\ifpdf
\hypersetup{
  pdftitle={STARK recovery of spatial transcriptomics images with ultra-low sequencing depths via adaptive regularization},
  pdfauthor={D. Doe, P. T. Frank, and J. E. Smith}
}
\fi

\begin{document}

\maketitle

\begingroup
\renewcommand\thefootnote{\fnsymbol{footnote}}
\footnotetext[1]{Department of Mathematics, University of British Columbia.}
\footnotetext[2]{Department of Computer Science, University of British Columbia.}
\footnotetext[3]{Department of Mathematics and Mathematical Statistics, Ume\r{a} University.}
\footnotetext[4]{Premium Research Institute for Human Metaverse Medicine (WPI-PRIMe), The University of Osaka.}
\footnotetext[5]{Co-first authors.}
\endgroup

\begin{abstract}
We present an approach to denoising spatial transcriptomics images that is particularly effective for uncovering cell identities in the regime of ultra-low sequencing depths, and also allows for interpolation of gene expression.
The method -- Spatial Transcriptomics via Adaptive Regularization and Kernels (STARK) -- augments kernel ridge regression with an incrementally adaptive graph Laplacian regularizer. In each iteration, we (1) perform kernel ridge regression with a fixed graph to update the image, and (2) update the graph based on the new image.  The kernel ridge regression step involves reducing the infinite dimensional problem on a space of images to finite dimensions via a modified representer theorem.  Starting with a purely spatial graph, and updating it as we improve our image makes the graph more robust to noise in low sequencing depth regimes. We show that the aforementioned approach optimizes a block-convex objective through an alternating minimization scheme wherein the sub-problems have closed form expressions that are easily computed. This perspective allows us to prove convergence of the iterates to a stationary point of this non-convex objective. Statistically, such stationary points converge to the ground truth with rate $\mathcal{O}(R^{-1/2})$ where $R$ is the number of reads.
In numerical experiments on real spatial transcriptomics data, the denoising performance of STARK, evaluated in terms of label transfer accuracy, shows consistent improvement over the competing methods tested.

\end{abstract}



\section{Introduction}
\label{sec: Introduction}

The substantial success of {\em spatial transcriptomics}~\cite{marx_MethodYearSpatially_2021} has led to massive new datasets, with high dimensional \emph{images of gene expression}. 
The tissue of interest is placed on a chip that captures RNA molecules from individual cells; material from each pixel is tagged with a unique \lq DNA barcode' that allows everything to be sequenced in parallel. 

Recent technological breakthroughs (e.g. Stereo-seq~\cite{chen_SpatiotemporalTranscriptomicAtlas_2022}) have dramatically increased the area of tissue which can be profiled and have achieved finer spatial resolutions.
As new technologies increase the scale of what can be profiled, the raw cost of sequencing RNA from cells over large areas will quickly become the limiting factor.
The cost of sequencing is paid on a \textit{per read} basis, where a \emph{read} refers to a captured RNA molecule being sequenced, and large scales typically require a large number of reads. 
Reducing the number of reads as a cost reduction strategy results in excessive noise in the measured gene expression data, including the well-known problem of \textit{dropouts} \cite{li_AccurateRobustImputation_2018}, where the observed \emph{counts matrix} contains excessive zero entries. Hence, there is a need to develop denoising algorithms for spatial transcriptomics.

In this work, we present an adaptive regularization scheme for denoising spatial transcriptomics images in the regime where the number of reads per pixel, also known as \emph{sequencing depth}, is extremely low---around 100 reads per pixel. Our method is based on kernel ridge regression and employs ideas from \citet{facciolo_exemplar-based_2009}, \citet{peyre_NonlocalRegularizationInverse_2011a} and \citet{pang_GraphLaplacianRegularization_2017} to achieve accurate image recovery with provable asymptotic guarantees.

\begin{figure}
    \centering
    \includegraphics[width=0.8\textwidth, trim=0 40 0 0, clip]{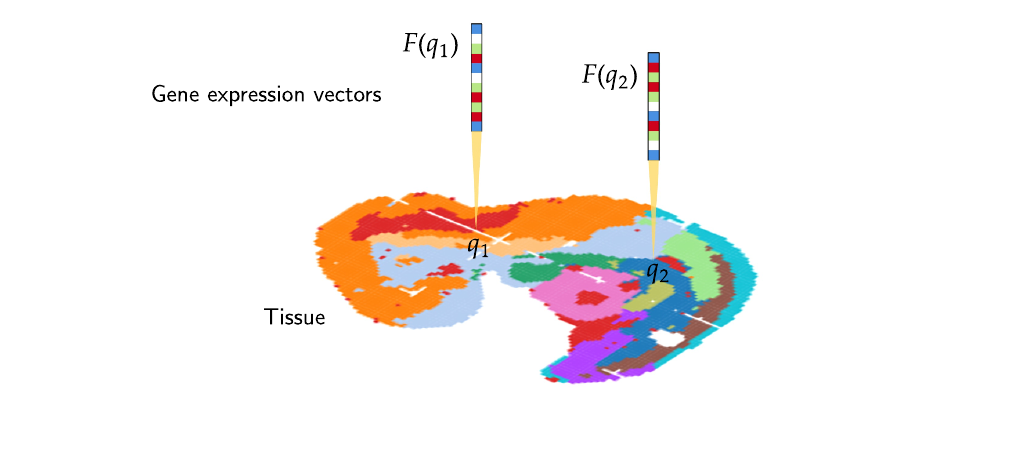}
    \caption{\textbf{Gene expression image}. Gene expression vectors at each spatial location in the tissue encode the relative abundances of mRNAs at those locations.}
    \label{fig:gene expression image}
\end{figure}

\subsection{Mathematical formulation}

 In this section, we formulate the problem of denoising a {gene expression image}.

\subsubsection{Gene expression images and denoising}
Our main object of interest is the transcriptional state of the cells in a tissue.
{Consider a spatial domain $\mathcal{Q}$ in $\mathbb{R}^{2}$ (with straightforward extensions to $\mathbb{R}^n, n\geq 3$). A gene expression image (or \textit{spatial transcriptomics image}) $F: \mathcal{Q} \to \mathbb{R}^d$ associates to each spatial location $q \in \mathcal{Q}$, a \emph{gene expression vector} $F(q) \in \mathbb{R}^d$, where $d$ denotes the number of genes in the genome\footnote{This is usually around $d \approx 10^4$.}; see \Cref{fig:gene expression image}. The vector $F(q)$ encodes the relative abundances of mRNA transcripts for each gene at $q$, and, in fact, lives in the probability simplex $\Delta_d \subseteq \mathbb{R}^d$ of non-negative vectors that sum to 1.}

Given a ground truth gene expression image $F^\star$ and a finite set of pixels $\{ q_{1}, \dots,q_{m} \} \subseteq \mathcal{Q}$, we obtain noisy measurements 
\begin{equation*}
Y_{i} = F^\star(q_{i}) + V_{i},
\end{equation*}
where $V_i$ denotes measurement noise. This gives a noisy gene expression matrix $\mathbf{Y} \in \mathbb{R}^{m \times d}$ whose rows are $Y_{1}, \dots, Y_{m}$. 
The \emph{denoising} problem can be summarized as
\begin{equation}
\begin{aligned}
\text{Estimate} \quad & F^\star(q_{1}), \dots, F^\star(q_{m}) \\
\text{from}\quad & Y_{1}, \dots, Y_{m}.
\end{aligned}\label{eq: denoising_problem}
\end{equation}
A more general problem is that of interpolating $F^\star(q)$ at arbitrary locations $q \in \mathcal{Q}$; here, one is effectively \emph{learning} $F^\star$ in the sense of nonparametric regression:
\begin{equation}
\begin{aligned}
\text{Estimate} \quad &F^\star \,\,\text{on } \mathcal{Q} \\
\text{from}\quad & Y_{1}, \dots, Y_{m}.
\end{aligned}\label{eq: interpolation_problem}
\end{equation}

Image denoising~\cite{rudin_NonlinearTotalVariation_1992, chambolle_FirstOrderPrimalDualAlgorithm_2011, buades_ReviewImageDenoising_2005} and nonparametric regression~\cite{geer_EmpiricalProcessesMEstimation_2000, bartlett_LocalRademacherComplexities_2005, tsybakov2003introduction} have been studied extensively in general settings, but there exist additional challenges in the context of spatial transcriptomics. Gene expression vectors are high-dimensional ($d \simeq 10^4$), and contain important biological information on cell identities that must be retained. Next, gene expression images are susceptible to high levels of multinomial noise in low-reads regimes~\cite{heimberg_LowDimensionalityGene_2016, townes_FeatureSelectionDimension_2019}, leading to very sparse observed matrices $\mathbf{Y}$. Lastly, on a more practical note, the pixel locations $\{q_1, \dots, q_m\}$ can be irregularly placed, depending on particular sequencing platforms such as 10x Genomics’ Visium~\cite{stahl_VisualizationAnalysisGene_2016}, Slide-seq~\cite{rodriques_SlideseqScalableTechnology_2019}. As a result, many image denoising techniques that work on matrix or tensor data (encoding a regular grid of pixels) do not work out of the box.

\subsubsection{Multinomial sequencing noise}\label{subsubsec: noise models}

Multinomial sequencing models are standard in the literature~\cite{heimberg_LowDimensionalityGene_2016, townes_FeatureSelectionDimension_2019}, and are closely related to the popular Poisson model~\cite{wang_GeneExpressionDistribution_2018, zhang_DeterminingSequencingDepth_2020}.
{The noisy expression vectors $Y_i$ are generated via a two-step procedure. In the first step, we select the number of measurements for each pixel:}
\begin{equation}
\label{eq: mult_pixelwise_reads}
(R_{1},\dots,R_{m}) \sim \mathrm{Multinomial}(R, \mathbf{u}),
\end{equation}
where $R$ denotes the total number of reads,  $R_i$ denotes the number of reads for pixel~$i$, and 
the vector $\mathbf{u}=(u_{1},\dots,u_{m}) \in \Delta_{m}$ models the normalized sampling frequencies of the pixels, specifying how often a read is sampled from each cell. 
In the second step, the gene-expression profiles $Y_i$ are generated conditional on $R_1,\ldots, R_m$ as follows:
\begin{align}
\label{eq: Y_given_reads}
Y_{i} \mid (R_{1},\dots,R_{m}) \sim \frac{1}{R_{i}}\indicator_{\{ R_{i}\geq 1 \}} \mathrm{Multinomial}(R_{i}, F^\star(q_{i})),
\end{align}
with $Y_{1},\dots,Y_{m}$ being independent conditional on the reads $(R_{1},\dots,R_{m})$. The vector $Y_i$ consists of empirically observed relative abundances of genes at pixel $q_i$, giving an empirical gene expression vector. The corresponding noise is then defined as 
\begin{align*}
V_{i}:=Y_{i}-F^\star(q_{i}),
\end{align*}
whose expected squared norm can be shown to be
\begin{align}
\label{eq: expected norm noise}
    \mathbb{E}\left[ \euc{Y_{i}-F^\star(q_{i})}^2 \mid (R_{1},\dots, R_m) \right] = \frac{1}{R_{i}} (1-\euc{F^\star(q_{i})}^2)
\end{align}
for $R_i \geq 1$. (Note that we use $\euc{\cdot}$ to denote Euclidean norm.)

\begin{remark}
The vector $\mathbf{u}$ of normalized sampling frequencies is unspecified, and may be thought of as a nuisance parameter. Regardless, our denoising algorithm is agnostic of $\mathbf{u}$ and simply assuming $u_{\min}:= \min_i u_i>0$ is sufficient for our asymptotic bounds. To achieve good denoising performance in practice however, we would like close-to-uniform sampling frequencies in the sense $u_i \geq c/m$ for some absolute constant $c>0$.
\end{remark}

\begin{remark}
    The two-step procedure above can also be packaged as one-step multinomial sampling, with an appropriately scaled matrix of parameters.
\end{remark}

\begin{remark}
Our results also allow for the Poisson noise model, which is obtained by replacing \eqref{eq: mult_pixelwise_reads} with $R_{i}\overset{\mathrm{iid}}{\sim}\mathrm{Poisson}(u_{i}R)$, while keeping \eqref{eq: Y_given_reads} the same. Here, $R$ would be the \textit{expected} total reads, instead of the actual number of reads which is stochastic.
\end{remark}

\subsection{Method overview}
\label{subsec: method_overview}
Our goal is to design a denoising strategy for problems~\eqref{eq: denoising_problem} and \eqref{eq: interpolation_problem} that spatially smooths the initially rough noisy image, but does so in an anisotropic manner to preserve inter-region boundaries between distinct cell types. Along these lines, we propose a variant of kernel ridge regression with additional graph Laplacian regularization that iteratively adapts to the structure of $F^\star$. This is realized by jointly optimizing an objective over the space of gene expression images $F$ and graph weights $W$ that determine the Laplacian regularizer. We describe the various components of this optimization problem below.

The image $F$ is constrained to $\mathcal H^d$, a \textit{vector-valued Reproducing Kernel Hilbert Space} (RKHS) with norm $\|\cdot\|_{\mathcal{H}^d}$ that we detail in \Cref{subsec: variational_formulation} (see~\eqref{eq: H_d_def}). To specify $W$, we first fix a (directed) graph $\mathcal{G}$ on the set of pixel locations $\{q_1,\dots,q_m\}$, with an edge set $\mathcal{E}_\tau$ ($\tau >0$) consisting of all pairs $(i,k) \in [m]\times[m]$ such that $\euc{q_i-q_k} \leq \tau$.
Then, the weights of all edges in $\mathcal{E}_\tau$ are described by $W$ which is constrained to the set
\begin{align}
    \label{eq: weights_set_C}
    \mathcal{C}_\tau := \left\{ W \in \mathbb{R}_+^{\mathcal{E}_\tau } \mid \textstyle\sum_{k:(i,k) \in \mathcal{E}_\tau} W_{(i,k)} = 1\,\, \forall i \in [m]\right\}.
\end{align}
We will always identify $W \in \mathcal{C}_\tau$ with a row-stochastic matrix $[W_{ik}]_{i,k=1}^m \in \mathbb{R}_+^{m \times m}$ such that $W_{ik}=0$ if $(i,k)\notin \mathcal{E}_\tau$.
The objective that we minimize is
\begin{equation}\label{eq: objective}
\begin{aligned}
    J(F,W) :=&\,\,  \underbrace{ \frac{1}{m} \sum_{i=1}^m \euc{ Y_i - F(q_i)}^2 }_{ \text{data fit} } + \underbrace{ \lambda \norm{F}_{\mathcal H^d}^2 }_{ \text{ridge regularization} }  \\
&+  \underbrace{ \omega\frac{1}{2m} \sum_{(i,k) \in \mathcal{E}_\tau}  W_{ik} \euc{F(q_i) - F(q_k)}^2 }_{ \text{graph Laplacian regularization} } + \underbrace{ \omega  E(W) }_{ \text{graph entropy} },
\end{aligned}
\end{equation}
where $E(\cdot)$ is an entropy functional (following \citet{facciolo_exemplar-based_2009}, \citet{peyre_NonlocalRegularizationInverse_2011a}) defined as
\begin{align*}
    E(W) := \frac{1}{2m} \sum_{(i,k) \in \mathcal{E}_\tau} \left\{s_{1}^2 \,W_{ik}\left( \log W_{ik}-1 \right) + \frac{s_{1}^2}{s_{2}^2}\,W_{ik} \euc{q_{i}-q_{k}}^2 \right\}.
\end{align*}
Note that $\lambda >0$, $\omega\geq 0$, and $s_1,s_2 >0$ are hyperparameters. Note also that $J(\cdot, \cdot)$ is block-wise convex, i.e. convex \emph{separately} in both $F$ and $W$. Let us tabulate the optimization problem as
\begin{align}
    \minimize_{F \in \mathcal{H}^d, W \in \mathcal{C}_\tau} \quad J(F,W). \tag{P} \label{eq: min_J_problem}
\end{align}
We briefly describe the roles of the four terms in~\eqref{eq: objective} below.

\paragraph{Data fit}
This term simply encourages the values of $F$ at the pixels $q_1, \dots q_m$ to be close to the noisy observations $Y_1, \dots, Y_m$ in a least-squares sense.\footnote{A likelihood-based loss function under of multinomial noise would be the Kullback-Leibler (KL) divergence. We found in experiments, however, that using the least squares loss leads to better denoising performance in certain biological meaningful metrics. It also results in simpler minimization algorithms.}

\paragraph{Ridge regularization}
Regularizing with the squared RKHS norm $\|\cdot\|_{\mathcal{H}^d}^2$ makes the objective $J$ strongly convex and coercive in $F$. Importantly, it also allows for a variant of the \emph{representer theorem}, which boils down the (potentially) infinite-dimensional problem of minimizing $J(\cdot, W)$ over $F \in \mathcal{H}^d$ to a tractable finite-dimensional problem. Lastly, depending on the particular choice of RKHS, the ridge regularization term can also promote smoothness in $F$.

\paragraph{Graph Laplacian regularization}
This term promotes smoothness in $F$ by getting nearby pixels $q_i, q_k$ to have similar gene expressions $F(q_i), F(q_k)$. Such smoothing can be done {anisotropically}, by using image gradient information, to prevent excessive blurring of inter-region boundaries. The graph weights $W$ are meant to learn and encode such anisotropy. 
In \Cref{subsec: disc Dirichlet Laplacian}, we interpret the graph Laplacian regularizer as a discretization of an anisotropic Dirichlet energy.

\paragraph{Graph entropy}
A typical way to adapt $W$ to the structure of $F$ is via Gaussian weights
\begin{align}
       W = \mathrm{RowNormalize}\left[ \exp \left( -|F(q_i)-F(q_k)|^2 /s_{1}^2 \right)  \exp \left( -|q_{i}-q_{k}|^2 /s_{2}^2 \right) \right]_{(i,k)\in \mathcal{E}_\tau}
       \label{eq: gaussian_weights_overview}
\end{align}
(here, $\mathrm{RowNormalize}(\cdot)$ scales the rows to sum to 1).
Directly prescribing such a nonlinear dependence of $W$ on $F$, however, makes the optimization problem nonconvex in $F$, and difficult to solve. On the other hand, incorporating the graph entropy term into the objective produces these Gaussian weights as an optimality condition, without prescribing any relationships a priori. More precisely, fixing $F$ and minimizing $J(F,\cdot)$ over $\mathcal{C}_\tau$ results in the weights~\eqref{eq: gaussian_weights_overview}.

The function $J(\cdot,\cdot)$ is in fact block-wise strongly convex on $\mathcal{H}^d \times \mathcal{C}_\tau$. \Cref{loc:statement_(asymptotic_rate_with_respect_to_the_reads)} in \Cref{sec: convergence_theory} shows that all stationary points $(\bar F, \bar W)$ of $J$ yield consistent estimators of the ground truth 
as the number of reads $R \uparrow \infty$. \Cref{thm: strong_convergence} on the other hand shows that the sequence of iterates $(F^t, W^t)_{t \in \mathbb{N}}$ generated by the alternating updates
\begin{equation}
\begin{aligned}
F^{t+1} &= \argmin_{F \in \mathcal H^d} J(F, W^{t}) \\
W^{t+1} &= \argmin_{W \in \mathcal C_\tau} J(F^{t+1},W),
\end{aligned}\tag{U} \label{eq: alternating_updates}
\end{equation}
does in fact converge to a stationary point $(\bar F, \bar W)$ of $J$. Finally, closed-form solutions are available for both subproblems resulting in \Cref{alg:iterative_update}.

\subsection{Roadmap}

In \Cref{sec: krr_theory} we give a mathematical foundation for our choice of objective function  \cref{eq: objective} based on minimal assumptions. We show that this objective occurs naturally when optimizing over an RKHS, and the additional regularizers impose a smoothness prior and ensure coercivity of the objective.
{\Cref{sec: convergence_theory} presents theoretical results on statistical and computational convergence, and provides closed-form expressions for the alternating updates. \Cref{sec: numerics} conducts numerical experiments on real spatial transcriptomics data, evaluating the performance of our method against SPROD~\cite{wang_SprodDenoisingSpatially_2022}, GraphPCA~\cite{yang_GraphPCAFastInterpretable_2024} and STAGATE~\cite{dong_DecipheringSpatialDomains_2022}. Finally, \Cref{sec: discussion} provides a concluding discussion.}

\subsection{Summary of contributions}
\label{subsec: contributions}
{Our main contributions are listed below.
\begin{itemize}
\item We develop the method Spatial Transcriptomics via Adaptive Regularization and Kernels (STARK) -- an iteratively adaptive approach for denoising spatial transcriptomics data that is particularly effective in low sequencing depth regimes. We place STARK into a broader context by connecting it to anisotropic smoothing and regularization.
\item We prove that the iterates converge to a stationary point of the objective, and establish asymptotic bounds on the statistical estimation error of such stationary points.
\item We propose a testing strategy for denoising performance on real spatial transcriptomics data, that is simple, general and practically relevant. Briefly, given a (high sequencing depth) counts matrix, we first \emph{downsample} it to significantly fewer reads. The downsampled counts matrix is fed into the denoising method being tested, and the output is compared with the original (normalized, fully-sampled) counts matrix.
\item {Our numerical experiments show STARK's improved denoising performance over competing methods as evaluated in biologically relevant metrics such as \emph{label transfer accuracy}, which measures the extent to which cell identities are preserved.}
\end{itemize}
}

\subsection{Related work}

We review related work on the denoising problem originating in single-cell analysis as well as in the more traditional field of image processing.

\paragraph{Denoising gene expression matrices}
There has been substantial work in the last few years on denoising or imputing counts matrices resulting from single-cell RNA sequencing~\cite{patruno_ReviewComputationalStrategies_2021}. 
MAGIC~\cite{vandijk_RecoveringGeneInteractions_2018} uses data diffusion to learn an underlying manifold in gene expression space, and performs denoising by effectively projecting the measured cells onto the manifold. 
SAVER~\cite{huang_SAVERGeneExpression_2018} employs
an empirical Bayes-like approach to denoising. {scImpute}~\cite{li_AccurateRobustImputation_2018} uses a targeted imputation strategy instead of globally denoising the counts matrix. 
McImpute~\cite{mongia_McImputeMatrixCompletion_2019} uses a slightly relaxed version of the targeted approach via low-rank matrix completion. More recent work has also used autoencoders \cite{eraslan_SinglecellRNAseqDenoising_2019, zhao_AETPGGNovelAutoencoderbased_2022} for denoising. Despite their success on single-cell RNA sequencing datasets however, these methods face a severe disadvantage in the context of spatial transcriptomics because they do not use any spatial information.

\paragraph{Denoising spatial transcriptomics images}
Pioneering work on denoising methods tailored to gene expression images with spatial structure was carried out by \citet{wang_SprodDenoisingSpatially_2022}. Their method SPROD uses a t-SNE-like procedure to learn a spatially-aware latent graph on the pixels. This graph then denoises the observed data via a regularized least squares step
\begin{equation}
\label{eq: sprod}
\minimize_{\mathbf{X}\in \bR^{m\times d}} \fro{\mathbf{Y}-\mathbf{X}}^2 + \omega \,\mathrm{tr}(\mathbf{X}^T\mathbf{L}\mathbf{X}),
\end{equation}
where $\mathbf{L}$ is the (unnormalized) graph Laplacian of the learnt latent graph. We realized though numerous experiments that graph Laplacian regularization as in~\eqref{eq: sprod} is very effective on real spatial transcriptomics data, often outperforming competing regularizers such as total variation. Additionally, its simplicity allows for closed-form solutions which are particularly convenient for quick testing and tuning even on large datasets. This provided motivation for including graph Laplacian regularization in our own objective function~\eqref{eq: objective}. 

SPROD also adapts its regularizer to the image $F^\star$, but unlike in the current work, employs a one-shot strategy: it directly uses the noisy gene expression matrix $\mathbf{Y}$ to learn its latent graph of pixel similarities. This strategy, which is similar to bilateral filtering~\cite{paris_BilateralFilteringTheory_2009}, generally performs well, but is naturally susceptible to problems in learning the graph when $\mathbf{Y}$ is very noisy. Intuitively, it is harder to detect inter-region boundaries and edges in noisier images. Our method, on the other hand, adapts incrementally by alternating between estimating $F^\star$ and adapting the weights $W$ to the current estimate of $F^\star$. Similar adaptive techniques have appeared in the image denoising and regularization literature in the form of iterative steering kernel regression~\cite{takeda_KernelRegressionImage_2007}, non-local regularization~\cite{peyre_NonlocalRegularizationInverse_2011a} and optimal graph Laplacian regularization~\cite{pang_GraphLaplacianRegularization_2017}, as well as in other areas such as terrain modelling~\cite{lang_AdaptiveNonStationaryKernel_} and optical flow estimation~\cite{middendorf_EmpiricallyConvergentAdaptive_2002}, providing a basis for this work. 
Traditional PDE-based methods on image smoothing and edge detection~\cite{perona_ScalespaceEdgeDetection_1990, catte_ImageSelectiveSmoothing_1992, tschumperle_VectorvaluedImageRegularization_2005} also contain notions of such incremental adaptivity.

Many recently developed methods for spatial transcriptomics denoising are based on dimensionality reduction and latent representations of gene expression vectors. Linear dimensionality reduction in the form of low-rank factorization is featured in SpatialPCA~\cite{shang_SpatiallyAwareDimension_2022}, GraphPCA~\cite{yang_GraphPCAFastInterpretable_2024}, nonnegative spatial factorization~\cite{townes_NonnegativeSpatialFactorization_2023} and the Master's thesis~\cite{kubal_TheoryAlgorithmsSpatial_2023}. GraphPCA in particular can be interpreted as a low-rank variant of SPROD, and benefits from the availability of closed-form solutions. MIST~\cite{wang_RegionspecificDenoisingIdentifies_2022} on the other hand uses a low-rank matrix completion approach following \citet{mongia_McImputeMatrixCompletion_2019}.
On the nonlinear side of latent representations, (graph-based) autoencoders play a key role in STAGATE~\cite{dong_DecipheringSpatialDomains_2022}, ADEPT~\cite{hu_ADEPTAutoencoderDifferentially_2023}, STGNNks~\cite{peng_STGNNksIdentifyingCell_2023}, SEDR~\cite{xu_UnsupervisedSpatiallyEmbedded_2024}, DenoiseST~\cite{cui_DenoiseSTDualchannelUnsupervised_2024} and DiffuST\cite{jiao_DiffuSTLatentDiffusion_2024}. We observed through experiments, however, that methods based on low-dimensional representations tend to perform worse in terms of label transfer accuracy, meaning that cell identities fail to be sufficiently preserved. This agrees with the results in \Cref{sec: numerics}, where we compare GraphPCA and STAGATE with SPROD and STARK.

\paragraph{Interpolating the image at new points}
Most existing denoising methods for spatial transcriptomics do not provide a direct scheme for estimating $F^\star$ at an arbitrary point $q \in \mathcal{Q}$ that is different from the measured pixels $q_1,\dots, q_m$. Notable exceptions include SpatialPCA~\cite{shang_SpatiallyAwareDimension_2022}, which uses distance kernels, and DIST~\cite{zhao_DISTSpatialTranscriptomics_2023}, which constructs a variational learning network. By framing our method through a variational problem over a space of functions $F$ (an RKHS specifically), we get evaluation at arbitrary points $q \in \mathcal{Q}$ essentially for free. 

\paragraph{Evaluating denoising performance}
Various testing strategies for denoising performance have been employed in the literature surveyed above, but no standardized framework appears to have emerged. \citet{wang_SprodDenoisingSpatially_2022} add iid Gaussian noise to a simulated gene expression matrix, which then goes through an exponential transformation and scaling step. The resulting noisy matrix is then denoised, and a relative $\ell^1$ denoising error is recorded. \citet{dong_DecipheringSpatialDomains_2022} simulate dropouts by randomly setting 30\% of the non-zero entries in a gene expression matrix to 0. After denoising, they compute the Pearson correlation coefficients of the denoised values to a reference ground truth. Recent work by \citet{cui_DenoiseSTDualchannelUnsupervised_2024} also simulates dropouts in a similar fashion, but evaluates clustering performance instead.
Here, we propose a denoising test based on downsampling reads that is simple, general and practically relevant (see \Cref{subsec: denoising_interpolation_tests}).

\subsection{Notation}

We summarize some of the notation used throughout the paper. First, $\euc{\cdot}$ and $\euc{\cdot}_{1}$ denote the Euclidean norm and $\ell^1$-norm of a vector respectively, $\op{\cdot}$ denotes the operator norm of a matrix, and $\langle \cdot,\cdot \rangle_{F}$ and $\fro{\cdot}$ are the Frobenius inner product and norm over matrices respectively. Similarly, $\langle \cdot,\cdot \rangle_{\mathcal{H}^d}$ and $\|\cdot\|_{\mathcal{H}^d}$ are the Hilbertian inner product and norm on the vector-valued Hilbert space $\mathcal{H}^d$.  
The trace function over matrices is written $\mathrm{tr}(\cdot)$, and the map $\text{RowNormalize}(\cdot)$ scales the rows of a matrix to sum to one, provided all rows are non-negative and non-zero. Given a matrix $\mathbf{M}$, its Moore-Penrose pseudoinverse is denoted by $\mathbf{M}^+$. Finally, given a function $F:\mathcal{Q}\to \mathbb{R}^d$ and the pixel locations $q_{1},\dots,q_{m}$, the corresponding boldface symbol $\mathbf{F}$ stands for an $m$-by-$d$ matrix with rows $F(q_{1}),\dots,F(q_{m})$.

\section{Kernel ridge regression with graph-Laplacian regularization}\label{sec: krr_theory}

We derive our method by first considering a fairly generic variational problem and specifying the components needed for its well-definedness and tractability. By considering a regularizer that depends both on space and gene expression information, we obtain an adaptive scheme for edge-aware denoising.

\subsection{Variational formulation of the regression problem} \label{subsec: variational_formulation}

We will start with a bare-bones variational problem geared towards recovering $F^\star$, and determine the structures necessary for well-definedness. Consider the following regularized problem
\begin{align}
    \minimize_{F \in \mathcal F} \frac{1}{m} \sum_{i =1}^m \euc{Y_i - F(q_i)}^2 + \omega S(F),
    \label{eq: Dirichlet reg}
\end{align}
where $\mathcal F$ is an appropriate choice of function class (which will end up being a vector-valued RKHS $\mathcal{H}^d$), and $\omega \geq 0$ is the regularization strength for the Dirichlet energy
\begin{align*}
S(F) = \|D F \|_{L^2(\mathcal Q)}^2 \equiv \sum_{j=1}^d\,\int_{\mathcal Q} \, \nabla F_{j}(q) \cdot \nabla F_{j}(q) dq.
\end{align*}
Here, $DF$ denotes the (weak) Jacobian of $F$.
Note that $S(\cdot)$ penalizes the gradients of the image $F$, thereby encouraging smoothness. This is closely related to Gaussian blurring~\cite{weickert_AnisotropicDiffusionImage_}. Smoothing isotropically can potentially blur edges, however. Following~\cite{perona_ScalespaceEdgeDetection_1990, weickert_AnisotropicDiffusionImage_, pang_GraphLaplacianRegularization_2017}, one can replace the Dirichlet energy above by its anisotropic counterpart
\begin{align}
S_{G}(F) := \sum_{j=1}^d\,\int_{\mathcal Q} \, \nabla F_{j}(q) \cdot G(q)^{-1} \nabla F_{j}(q)\,dq,\label{eq: anisotropic-dirichlet-2}
\end{align}
where $G:\mathcal{Q}\to \mathbb{R}^{2\times2}$ is a matrix-valued function such that $G(q)$ is positive definite for all $q \in \mathcal{Q}$. $G$ should depend on local image gradients, see \Cref{subsec: disc Dirichlet Laplacian}.

To determine the function class $\mathcal F$, we examine the requirements for a well-defined problem (see also \citet{green_MinimaxOptimalRegression_2021}). Firstly, since the Dirichlet energy depends on first-order derivatives, we need $F \in \mathcal{F}$ to be at least (weakly) once-differentiable. This suggests the Sobolev space
\begin{align*}
H^1(\mathcal Q; \mathbb{R}^d) = \left\{F \in L^2(\mathcal{Q}; \mathbb{R}^d) \mid D F \in L^2(\mathcal{Q}; \mathbb{R}^{d \times 2}) \right\}.
\end{align*}
However, in $H^1(\mathcal Q; \mathbb{R}^d)$, elements $F$ need not have a continuous representative, and so pointwise evaluation $q\mapsto F(q)$ need not be well-defined everywhere
\cite{green_MinimaxOptimalRegression_2021}, which makes the problem \eqref{eq: Dirichlet reg} ill-defined. 
This can be resolved by imposing additional smoothness via choosing the space $H^s(\mathcal{Q}; \mathbb{R}^d)$ for $s > 1$, where up to $s$ (weak) derivatives of $F$ are required to be in $L^2$. If $s>1$, pointwise evaluation \emph{is} in fact well-defined (and bounded) in $H^s(\mathcal{Q}; \mathbb{R}^d)$.

In moving to $H^s(\mathcal{Q}; \mathbb{R}^d)$ however, the objective function in~\eqref{eq: Dirichlet reg} loses coercivity, and hence, may not admit minimizers. This can be corrected by regularizing additionally with a squared Sobolev norm $\|\cdot\|_{H^s}^2$ (which also provides strong convexity), to give the \textit{spline}-like problem
\begin{align}
    \minimize_{F \in H^s(\mathcal{Q}, \mathbb{R}^d)} \quad \frac{1}{m}\sum_{i=1}^m \euc{Y_{i}-F(q_{i})}^2 + \omega\, S_G(F) + \lambda\, \|F\|_{H^s}^2,
    \label{eq: Dirichlet spline}
\end{align}
for a hyperparameter $\lambda>0$. Another advantage of working in the space $H^s(\mathcal{Q}; \mathbb{R}^d)$, given $s>1$, is that it is a (vector-valued) \emph{Reproducing Kernel Hilbert Space} (RKHS)~\cite{aronszajn_TheoryReproducingKernels_1950}, which is effectively generated by a single positive semidefinite kernel function $\mathcal{K}:\mathcal{Q}\times \mathcal{Q} \to \mathbb{R}$. This injects computational tractability into the problem.
Sobolev spaces $H^s$ are generated by Mat\'{e}rn kernels~\cite{porcu_MaternModelJourney_2023} (depending on $s$), where the case $s=3/2$ in particular corresponds to a simple exponential kernel $\mathcal{K}(q,q')=\exp(-\euc{q-q'}/\lsc)$ where $\lsc>0$.   (Precise statements and technical details regarding domain restrictions are provided in \Cref{appsec: theory-details}.)

From here on, we would like to abstract away from particular Sobolev spaces and work directly with some general (vector-valued) RKHS.

\paragraph{Vector-valued Reproducing Kernel Hilbert Spaces}
A vector-valued RKHS~\cite{schwartz_SousespacesHilbertiensDespaces_1964, micchelli_LearningVectorValuedFunctions_2005} generalizes its scalar counterpart by considering functions into a vector space such as $\mathbb R^d$ as opposed to just scalar-valued functions. Given a scalar-valued RKHS $\mathcal H$, the Cartesian product space 
\begin{equation}\label{eq: H_d_def}
\mathcal{H}^{d} = \{ F: \mathcal{Q} \to \mathbb{R}^d \mid F_{j} \in\mathcal{H} \text{ for all } j \in[d]  \},
\end{equation}
with inner product and norm
\begin{equation*}
\langle F,\tilde{F} \rangle_{\mathcal{H}^d} := \textstyle\sum_{j=1}^d \langle F_{j}, \tilde{F}_{j} \rangle_{\mathcal{H}}, \qquad \|F\|_{\mathcal{H}^d} = \langle F, F \rangle_{\mathcal{H}^d}^{1/2},
\end{equation*}
is a vector-valued RKHS. Now, we replace $H^s(\mathcal{Q}; \mathbb{R}^d)$ in problem~\eqref{eq: Dirichlet spline} with $\mathcal{H}^d$; the corresponding Sobolev norm $\|\cdot\|_{H^s}$ should then be replaced by $\|\cdot\|_{\mathcal{H}^d}$. This provides great computational advantages, as we may now use variants of \emph{kernel ridge regression} to attack the problem (given a suitable discretization of $S_G$).

\subsection{Adaptive Dirichlet energy and discretization}
\label{subsec: disc Dirichlet Laplacian}

So far, we have formulated an optimization problem~\eqref{eq: Dirichlet spline} with a precise choice of $\mathcal{F}$, and what remains is to choose $G$ in the anisotropic Dirichlet energy $S_G$ defined in \eqref{eq: anisotropic-dirichlet-2}.
The key function of the anisotropic Dirichlet energy term is to smooth while preserving image edges or inter-regional boundaries by making use of information on local image gradients. Suppose an oracle returns image gradients of the ground truth $F^\star$; then, one can set $G$ to be the corresponding \emph{oracle} structure tensor~\cite{tschumperle_VectorvaluedImageRegularization_2005, pang_GraphLaplacianRegularization_2017}
\begin{align}\label{eq: oracle_structure_tensor}
    G^{\star}(q) := I_{2}+ \tfrac{s_1^2}{s_2^2}DF^{\star}(q)^TDF^{\star}(q)
\end{align}
where $I_2$ is the 2-by-2 identity matrix, and $s_1, s_2 >0$ are hyperparameters.

\paragraph{Approximating the Dirichlet energy with the graph Laplacian regularizer}
Recall the (directed) graph $\mathcal{G}$ from \Cref{subsec: method_overview} with vertices $\{q_1, \dots, q_m\}$ and edge set $\mathcal{E}_\tau$. Given edge weights $W_{ik}$, for $(i,k)\in \mathcal{E}_\tau$, consider the graph Laplacian regularizer
\begin{align}
    \widehat{S}_\mathcal{G}(F) = \tfrac{1}{2m}\textstyle\sum_{(i,k) \in \mathcal E} W_{ik}\,\euc{F(q_k)-F(q_i)}^2.
    \label{eq:graph_lap_regularizer}
\end{align}
\citet[Theorem 1]{pang_GraphLaplacianRegularization_2017} show that $\widehat{S}_\mathcal{G}$ provides a good approximation to $S_{G^\star}$, under certain conditions, when $m$ is large. The underlying principle is that the graph $\mathcal{G}$ with weights\footnote{Note that the weights $W^\star$ here follow a slightly different normalization than that in \citet{pang_GraphLaplacianRegularization_2017}; row-normalized weights prove useful in the context of our algorithm.}
\begin{align*}
       W^\star = \mathrm{RowNormalize}\left[ \exp \left( -|F^\star(q_i)-F^\star(q_k)|^2 /s_{1}^2 \right)  \exp \left( -|q_{i}-q_{k}|^2 /s_{2}^2 \right) \right]_{(i,k)\in \mathcal{E}_\tau}
\end{align*}
serves as a discrete approximation to the Riemannian manifold
\begin{align*}
    \mathrm{graph}\left(\tfrac{s_1}{s_2} F^\star\right) = \left\{\left(q,\, \tfrac{s_1}{s_2}F^\star(q)\right) \mid  q \in \mathcal{Q} \right\},
\end{align*}
whose Dirichlet energy functional is closely linked to $S_{G^\star}$~\cite{belkin_TheoreticalFoundationLaplacianbased_2008, hein_UniformConvergenceAdaptive_2006}.
Following this discretization, we end up replacing \eqref{eq: Dirichlet spline} with the problem,
\begin{align}
    \minimize_{F \in \mathcal{H}^d} \quad \frac{1}{m}\sum_{i=1}^m \euc{Y_{i}-F(q_{i})}^2 + \frac{\omega}{2m}\sum_{(i,k) \in \mathcal E} W^{\star}_{ik}\,\euc{F(q_k)-F(q_i)}^2 + \lambda\, \|F\|_{\mathcal{H}^d}^2.
\end{align}
Importantly, constructing $W^\star$ would require information about the ground truth $F^\star$ provided by an oracle (as is the case for the oracle structure tensor $G^\star$). In practice, one needs to estimate $W^\star$. 
Thus, following \citet{peyre_NonlocalRegularizationInverse_2011a}, we can iteratively adapt $W$ to the current denoised estimate $F$ as per \eqref{eq: gaussian_weights_overview}, update the associated regularizer, and then resolve the optimization problem under the new regularizer. This alternating procedure is, in fact, equivalent to performing block coordinate minimization on \cref{eq: objective}. We demonstrate this in the next section.

\section{Theoretical results}\label{sec: convergence_theory}

We now present some theoretical results regarding the statistical and computational convergence of our denoising method. 
Recall from \Cref{subsec: variational_formulation}, the vector-valued Reproducing Kernel Hilbert Space (RKHS) $\mathcal{H}^d = \{F:\mathcal{Q}\to \mathbb{R}^d \mid F_j \in \mathcal{H},\, \forall j \in [d]\}$ generated from a scalar-valued RKHS $\mathcal{H}$, and let $\mathcal{K}:\mathcal{Q}\times \mathcal{Q} \to \mathbb{R}$ denote the positive semidefinite kernel function for this RKHS. 
Our statistical convergence result states that stationary points of the objective function $J$ of~\eqref{eq: min_J_problem} converge to the ground truth, while the computational convergence guarantees that the alternating updates \eqref{eq: alternating_updates} converge to a stationary point of $J$. 

\begin{definition} A point $(\bar F, \bar W) \in   \mathcal H^d \times \mathcal C_\tau$ is a stationary point of $J$ over $\mathcal H^d \times \mathcal C_\tau$ if 
\begin{align*}
    -\nabla_W J(\bar F, \bar W) \in N_{\mathcal{C}_\tau}(\bar W), &\qquad \nabla_F J(\bar F, \bar W) = 0,
\end{align*}
where $N_{\mathcal{C}_\tau}(\bar{W})$ is the normal cone to ${\mathcal{C}_\tau}$ at $\bar{W}$. Let $\mathrm{Stat}(J) \subseteq \mathcal{H}^d \times \mathcal{C}_\tau$ denote the set of all such stationary points.
\label{def: stat}
\end{definition}

\subsection{Convergence}

We start with statistical convergence, showing that as the reads $R \uparrow \infty$, all stationary points of $J$ approach the ground truth $F^\star$ in the empirical $L^2$ norm defined via
\begin{align*}
    \|F\|_{L^2_{m}}^2 := \tfrac{1}{m}\textstyle\sum_{i=1}^m \euc{F(q_{i})}^2.
\end{align*}
Notice that this empirical $L^2$ norm is simply the scaled Frobenius norm of a matrix $\mathbf{F}$ with rows $F(q_1), \dots, F(q_m)$, i.e. $\|F\|_{L^2_m}= \fro{\mathbf{F}}/\sqrt{m}$.

Suppose the measurements $Y_{1},\dots,Y_{m}$ are generated according to the noise model from \Cref{subsubsec: noise models}, with $R$ total reads. The precise statement is as follows.

\begin{theorem}[Asymptotic rate with respect to the reads]
\label{loc:statement_(asymptotic_rate_with_respect_to_the_reads)}
Consider the objective function $J$ from \cref{eq: objective} with strictly positive hyperparameters $\lambda=\mathcal{O}(R^{-1})$ and $\omega=\mathcal{O}(R^{-1})$. We have
\begin{align*}
    \mathbb{E}\,\sup \{ \|\bar{F} - F^\star\|_{L^2_{m}} \mid (\bar{F}, \bar{W}) \in \mathrm{Stat}(J) \} = \mathcal{O}(R^{-1/2}).
\end{align*}
\end{theorem}
\hyperlink{loc:statement_(asymptotic_rate_with_respect_to_the_reads).proof}{The proof} relies on several definitions and results to follow, and is presented towards the end of the paper in \Cref{sec: proofs}.
\Cref{loc:statement_(asymptotic_rate_with_respect_to_the_reads)} provides a notion of consistency as the reads $R \uparrow \infty$, or in other words, the noise level goes to 0 by \eqref{eq: expected norm noise}. {Note also that this result is \emph{asymptotic} -- the constant hidden in $\mathcal{O}(\cdot)$, while fixed with respect to $R$, is problem-dependent in general.}

The empirical $L^2$ norm $\|\cdot\|_{L^2_m}$ as an error metric in \Cref{loc:statement_(asymptotic_rate_with_respect_to_the_reads)}, although standard, is concerned only with the behaviour of $\bar{F}$ at the points $q_1, \dots, q_m$ where measurements $\mathbf{Y}_1,\dots, \mathbf{Y}_m$ are available. A more global metric that is natural in our setting is provided by the RKHS norm $\|\cdot\|_{\mathcal{H}^d}$. The corollary to follow obtains a similar bound in $\|\cdot\|_{\mathcal{H}^d}$, but with an additional interpolation error term. To set it up, first define the finite-dimensional subspace
\begin{align*}
M(q_{1},\dots,q_{m}) := \left\{ \tfrac{1}{\sqrt{ m }} \textstyle\sum_{i=1}^m \mathcal{K}(q_{i},\cdot)\theta_{i} \mid \theta_{1},\dots,\theta_{m} \in \mathbb{R}^d \right\} \subseteq \mathcal{H}^d
\end{align*}
(\Cref{loc:statement_(representer_theorem)} to follow implies that for any stationary point $(\bar{F}, \bar{W})$, $\bar{F}$ lies in this subspace). The approximation (or interpolation) error of $F^\star$ with respect to this subspace is defined as
\begin{equation*}
    \mathcal{A}(F^\star; \{ q_{1},\dots,q_{m} \}) := \inf \left\{ \|F^\star - F\|_{\mathcal{H}^d} \mid F \in M(q_{1},\dots,q_{m}) \right\}.
\end{equation*}

\begin{corollary}[Error in the RKHS norm]
\label{loc:statement_(error_in_the_rkhs_norm)}
Consider the objective function $J$ with strictly positive hyperparameters $\lambda=\mathcal{O}(R^{-1})$ and $\omega=\mathcal{O}(R^{-1})$. We have
\begin{equation*}
\mathbb{E}\,\sup \{ \|\bar{F} - F^\star\|_{\mathcal{H}^d} \mid (\bar{F}, \bar{W}) \in \mathrm{Stat}(J) \} \leq \mathcal{A}(F^\star; \{ q_{1},\dots,q_{m} \}) + \mathcal{O}(R^{-1/2}).
\end{equation*}
\end{corollary}
See \hyperlink{loc:proof_(error_in_the_rkhs_norm)}{the proof} in \Cref{appsec: deferred-proofs}. This error bound consists of an approximation error and an estimation error. The approximation error $\mathcal{A}(F^\star; \{ q_{1},\dots,q_{m} \})$ would depend on the number of pixels $m$ and their placement in $\mathcal{Q}$. We do not attempt to further quantify it here.

\begin{remark}
{The error bound of \Cref{loc:statement_(error_in_the_rkhs_norm)} has a couple of interesting consequences. First, for any $q \in \mathcal{Q}$, it holds by RKHS properties that $\euc{\bar{F}(q)-F^\star(q)} \leq C_{\mathcal{K}}\|\bar{F}-F^\star\|_{\mathcal{H}^d}$,
where the constant $C_{\mathcal{K}}:= \sup_{q \in \mathcal{Q}} \mathcal{K}(q,q)^{1/2}$. This, together with \Cref{loc:statement_(error_in_the_rkhs_norm)}, leads to some simple pointwise estimates of the error. Second, we obtain estimates on the RKHS norm of $\bar{F}$ via the triangle inequality as
\begin{equation*}
\begin{aligned}
\mathbb{E}\sup \{\|\bar{F}\|_{\mathcal{H}^d} \mid (\bar{F}, \bar{W}) \in \mathrm{Stat}(J)\} &\leq \|F^\star\|_{\mathcal{H}^d} + \mathcal{A}(F^\star; \{ q_{1},\dots,q_{m} \}) + \mathcal{O}(R^{-1/2}) \\
&\leq 2 \|F^\star\|_{\mathcal{H}^d} + \mathcal{O}(R^{-1/2}).
\end{aligned}
\end{equation*}
Such estimates can be interpreted as quantifying the expected \textit{regularity} of $\bar{F}$.
}
\end{remark}

Notice that the suprema in the bounds of \Cref{loc:statement_(asymptotic_rate_with_respect_to_the_reads)} and  \Cref{loc:statement_(error_in_the_rkhs_norm)} are over \textit{all} stationary points of the objective $J$, including, in view of \Cref{thm: strong_convergence}, the limiting point returned by the updates \eqref{eq: alternating_updates}. 
We remark that this is not very surprising, mainly because the input noisy counts themselves satisfy $\mathbb{E}\fro{\mathbf{Y}-\mathbf{F}^\star} = \mathcal{O}(R^{-1/2})$, and the updates have certain stability properties as can be inferred from the proofs.

The convergence of the iterates generated by \eqref{eq: alternating_updates} is established by the following result.
\begin{theorem}\label{thm: strong_convergence}
The sequence $(F^{t}, W^{t})_{t \in \mathbb N}$ generated by the updates \eqref{eq: alternating_updates} converges to a stationary point of $J$ over $\mathcal{H}^d \times \mathcal{C}_\tau$.
\end{theorem}
\hyperlink{prf: proof_of_strong_convergence}{The proof} is presented in \Cref{prfsec: strong_convergence}.

\subsection{Updates of the block coordinate descent algorithm}

In this section, we demonstrate that both of the updates in \eqref{eq: alternating_updates} have closed-form expressions. In particular, for the $F$-update, we establish a useful lemma that identifies the infinite-dimensional problem with an associated finite-dimensional problem. To state these results, one needs to define a few objects.

Given pixel locations $q_1,\dots,q_m$, and the kernel function $\mathcal{K}(\cdot, \cdot)$ for our RKHS, the corresponding \emph{kernel matrix} is a symmetric, positive semidefinite matrix defined as
\begin{align}
    \mathbf{K} := \left[ \tfrac{1}{m} \mathcal{K}(q_i, q_k) \right]_{i,k=1}^m  \quad\in \mathbb{R}^{m \times m}.
\end{align}
The following subspace of matrices corresponding to $\mathbf{K}$ will play a role later:
\begin{align*}
    \mathbf{K}\mathbb{R}^{m \times d} \equiv \{ \mathbf{K}\mathbf{X} \mid  \mathbf{X} \in \mathbb{R}^{m \times d}\}.
\end{align*}
Next, given a row-stochastic weights matrix $W \in \mathcal{C}_\tau$, define the effective graph Laplacian
\begin{align}
\label{eq: effective_graph_Laplacian}
    \overline{\mathbf{L}}_{W} := \frac{1}{2}({I}_{m}+\mathrm{Diag}(W^T\mathbf{e})) - \frac{1}{2} \left( W+W^T \right)  \quad\in \mathbb{R}^{m \times m},
\end{align}
where $\mathbf{e}=(1,\dots,1) \in \mathbb{R}^m$. This is the usual unnormalized graph Laplacian of the symmetric weights matrix $\frac{1}{2}(W + W^T)$, and satisfies
\begin{align}
\label{eq: lap-trace}
    \frac{1}{2m}\sum_{ik}W_{ik}\euc{F(q_{k})-F(q_{i})}^2 = \frac{1}{m} \mathrm{tr}(\mathbf{F}^T \overline{\mathbf{L}}_{W}\mathbf{F}).
\end{align}
Consequently, $\overline{\mathbf{L}}_W$ is also symmetric and positive semidefinite. Finally, define the finite-dimensional counterpart to $J$, denoted by $\mathbf{J}:\mathbb{R}^{m\times d} \times \mathbb{R}_+^{\mathcal{E}_\tau} \to \mathbb{R}$, as follows:
\begin{equation}\label{eq: bold_J}
\mathbf{J}(\boldsymbol{\theta},W):= \frac{1}{m} \fro{\mathbf{Y}-\sqrt{ m }\mathbf{K}\boldsymbol{\theta}}^2 + \lambda\, \mathrm{tr}\left( \boldsymbol{\theta}^T \mathbf{K} \boldsymbol{\theta}\right) + \omega\left\{ \mathrm{tr}\left( \boldsymbol{\theta}^T \mathbf{K}\overline{\mathbf{L}}_{W}\mathbf{K}\boldsymbol{\theta} \right) + E(W) \right\}.
\end{equation}
We have a variant of the \emph{representer theorem} from RKHS theory, \hyperlink{loc:proof_(representer_theorem)}{proved} in \Cref{appsec: deferred-proofs}.

\begin{lemma}[A representer theorem]
\label{loc:statement_(representer_theorem)}
Given $W \in \mathcal{C}_{\tau}$, $\lambda >0$, $\omega \geq 0$, a function $\hat{F}\in \mathcal{H}^{d}$ satisfies $\hat{F} = \argmin_{F \in \mathcal{H}^{d}}\,J(F,W)$
if and only if $\hat{F}$ can be represented as $\hat{F}(\cdot) = \frac{1}{\sqrt{ m }} \sum_{i=1}^m \mathcal{K}(q_{i}, \cdot)\hat{\theta}_{i}$,
where the vectors $\hat{\theta}_{1},\dots,\hat{\theta}_{m} \in \mathbb{R}^d$ are the rows of the unique matrix $\hat{\boldsymbol{\theta}} \in \mathbb{R}^{m\times d}$ satisfying
\begin{equation}
\begin{aligned}
\label{eq:theta_update}
\hat{\boldsymbol{\theta}} = \argmin_{\boldsymbol{\theta} \in \mathbf{K}\mathbb{R}^{m\times d}}\,\mathbf{J}(\boldsymbol{\theta}, W) = \frac{1}{\sqrt{ m }}(\mathbf{K}^2+\lambda \mathbf{K}+\omega \mathbf{K}\overline{\mathbf{L}}_{W}\mathbf{K})^+ \mathbf{K}\mathbf{Y}.
\end{aligned}
\end{equation}
\end{lemma}
See \hyperlink{loc:proof_(representer_theorem)}{the proof} in \Cref{appsec: deferred-proofs}.

\begin{remark}
\label{loc:remark_(representer_theorem)}
The function $\boldsymbol{\theta} \mapsto \mathbf{J}(\boldsymbol{\theta}, W)$ is strongly convex on $\mathbf{K}\mathbb{R}^{m\times d}$ with modulus (at least) $\Lambda_{r}(\mathbf{K})^2 + \lambda\,\Lambda_{r}(\mathbf{K})$, where $\Lambda_{r}(\mathbf{K})$ denotes the smallest strictly positive eigenvalue of $\mathbf{K}$. 
Note, additionally, that $\hat{\boldsymbol{\theta}}$ in \eqref{eq:theta_update} also minimizes $\mathbf{J}(\cdot,W)$ over the whole space $\boldsymbol{\theta} \in \mathbb{R}^{m\times d}$.
\end{remark}

Given an iterate $W^t \in \mathcal C_\tau$, we obtain $F^{t+1}$, the updated $F$-iterate from \eqref{eq: alternating_updates}, using \Cref{loc:statement_(representer_theorem)}.
Now that we have a closed-form expression for $F^{t+1}$, we can go on to solve for $W^{t+1}=\argmin_{W \in \mathcal{C}_\tau} \,J(F^{t+1}, W)$. 
The following proposition establishes that the weighting procedure introduced in \cref{eq: gaussian_weights_overview} is equivalent to solving this $W$-update subproblem. We remark that the additional entropy term $E(W)$ is required in order to establish this equivalence. 

\begin{proposition}\label{prop: W_min_solution}
Let $F \in \mathcal{H}^d$, $\lambda>0$ and $\omega>0$. The function $W \mapsto J(F,W)$ is $(\omega /2m)$-strongly convex on $\mathcal{C}_{\tau}$, and the unique minimizer $\hat{W} = \argmin_{W \in \mathcal{C}_\tau}\, J(F, W)$
is given by $\hat{W}=\mathrm{RowNormalize}(\tilde{W})$, where
\begin{equation}\label{eq: W_update}
\tilde W_{ik} := \begin{cases}
\exp \left( - \frac{1}{s_1^2} \euc{F(q_k) - F(q_i)}^2 \right) \exp \left( - \frac{1}{s_2^2} \euc{q_k - q_i}^2\right), & \text{if } |q_k - q_i | \leq \tau, \\
0 & \text{otherwise.}
\end{cases}
\end{equation}
\end{proposition}
See \hyperlink{loc:proof_of_W_min_solution}{the proof} in \Cref{appsec: deferred-proofs}. 

The above results, together with an initialization choice $W^0$, are summarized into \Cref{alg:iterative_update}.
\begin{algorithm}
\caption{The alternating updates of STARK}
\label{alg:iterative_update}
\begin{algorithmic}[1]
\State Set $\tilde{W}^0 \gets \left[ \exp \left( -|q_{k}-q_{i}|^2 /s_{2}^2 \right) \,\indicator_{\{ |q_{k}-q_{i}| \leq \tau \}} \right]_{i,k\in[m]}$
\State $W^0 \gets \text{RowNormalize}(\tilde{W}^0)$
\For{$t=0,\dots,N-1$}
    \State $\boldsymbol{\theta}^{t+1} \gets \frac{1}{\sqrt{ m }}(\mathbf{K}^2+\lambda \mathbf{K}+\omega \mathbf{K}\overline{\mathbf{L}}_{W^{t}}\mathbf{K})^+ \mathbf{K}\mathbf{Y}$
    \State $\mathbf{F}^{t+1} \gets \sqrt{ m }\mathbf{K}\boldsymbol{\theta}^{t+1}$
    \State $\tilde{W}^{t+1} \gets \left[ \exp \left( -|\mathbf{F}^{t+1}_{k\cdot}-\mathbf{F}^{t+1}_{i\cdot}|^2 /s_{1}^2 \right)  \exp \left( -|q_{k}-q_{i}|^2 /s_{2}^2 \right) \,\indicator_{\{ |q_{k}-q_{i}| \leq \tau \}} \right]_{i,k\in[m]}$
    \State $W^{t+1} \gets \text{RowNormalize}(\tilde{W}^{t+1})$
\EndFor
\State \Return $\boldsymbol{\theta}^N, W^N$ and the function $F^{N}(\cdot) = \frac{1}{\sqrt{ m }} \sum_{i=1}^m \mathcal{K}(q_{i}, \cdot)\boldsymbol{\theta}^N_{i}$ where $\boldsymbol{\theta}_{1}^N,\dots,\boldsymbol{\theta}_{m}^N \in \mathbb{R}^d$ are the rows of $\boldsymbol{\theta}^N$.
\end{algorithmic}
\end{algorithm}
As currently presented, the algorithm fixes the total number of iterations $N$ beforehand instead of using an optimality-based stopping criterion. This works well in practice because the first iterate $F^1 = \argmin_{F \in \mathcal{H}^d} J(F, W^0)$ is usually already a decent estimate of the ground truth $F^\star$ (see the discussion on the \textit{spatial} variant in \Cref{subsec: num_results} and \Cref{fig:spatial_oracle_comparisons}), so that a few more iterations suffice to produce a good result and one does not need to wait for convergence.
Nevertheless, it is possible to assess stationarity and convergence, in reference to \Cref{def: stat}, by tracking the distance between successive iterates as follows.
\begin{proposition}[Stationarity measure]
\label{res: bounding-non-stationarity}
For any $t \in \mathbb{N}$, we have
\begin{align*}
-\nabla_{W}J(F^t, W^t) \in N_{\mathcal{C}_{\tau}}(W^t), &\qquad
\|\nabla_{F} J(F^t, W^t)\|_{\mathcal{H}^d} \leq B \|F^{t+1}-F^t\|_{\mathcal{H}^d},
\end{align*}
where the constant $B >0$ does not depend on $t$.
\end{proposition}
See \hyperlink{prf: bounding-non-stationarity}{the proof} in \Cref{appsec: deferred-proofs}.
Via the variable $\boldsymbol{\theta}$, we can easily compute $\|F^{t+1}-F^t\|_{\mathcal{H}^d}^2= \mathrm{tr}\,[(\boldsymbol{\theta}^{t+1}-\boldsymbol{\theta}^t)^T \mathbf{K}(\boldsymbol{\theta}^{t+1}-\boldsymbol{\theta}^t)]$.

\section{Numerical experiments}\label{sec: numerics}
The recently-developed Stereo-seq technology \cite{chen_SpatiotemporalTranscriptomicAtlas_2022} has produced massive, high-quality spatial transcriptomics datasets, including the Mouse Organogenesis Spatiotemporal Transcriptomic Atlas (MOSTA). This dataset consists of gene expression image snapshots of mouse embryos at various stages of development. The high quality of the data is owed to a large field of view, cellular-scale resolution and a larger number of total reads, as compared to previous technologies.

We report on several numerical experiments evaluating our method {STARK} on the MOSTA dataset, including comparisons with SPROD \cite{wang_SprodDenoisingSpatially_2022}, GraphPCA~\cite{yang_GraphPCAFastInterpretable_2024} and STAGATE~\cite{dong_DecipheringSpatialDomains_2022}.
STARK is implemented in python, and is available at \url{https://github.com/schiebingerlab/STARK}. {The inputs to STARK are} (1) a pixels-by-genes noisy gene expression matrix $\mathbf{Y}$ whose rows are normalized to sum to 1, (2) spatial $(x,y)$ coordinates specifying the locations of the pixels $q_1,...,q_m$ from which the reads were observed, and optionally, (3) the number of reads $R_1, \dots, R_m$ sequenced from each of the pixels, which help auto-tune regularization strengths. 
{We do any additional transformations on the counts, e.g. \texttt{log1p}, \textit{after} denoising.}

The main output of the algorithm (see \Cref{alg:iterative_update}) is a function $\bar F: \mathcal{Q} \to \mathbb{R}^d$, which can be evaluated at any $q \in \mathcal{Q}$. Given spatial coordinates specifying (possibly new) locations $q_1',\dots,q'_{m'}$, the gene expression vectors at those locations can be estimated as $\bar F(q_1'), \dots, \bar F(q'_{m'})$. Noting that gene expression vectors should ideally lie in the simplex $\Delta_d$, we additionally threshold $\bar F(q_1'), \dots, \bar F(q'_{m'})$ to the simplex via the map $X \mapsto {X_+}\,/\,{|X_+|_1}$,
where $X_+:= \max(X,0)$ is defined component-wise. This thresholding map is more of a simple heuristic and can be replaced by Euclidean (or other) projections onto $\Delta_d$ if desired. 

Details on suitable hyperparameter settings for STARK are provided in \Cref{appsubsec: hyperparameters}.
Importantly, in the experiments to follow, we pick the exponential kernel $\mathcal{K}(q, q'):= \exp(-|q'-q|/\lsc)$, where $\lsc>0$ is a tunable length-scale. Recall that for this choice of kernel, the resulting vector-valued RKHS $\mathcal{H}^d$ coincides with the Sobolev space $H^{3/2}(\mathcal{Q};\mathbb{R}^d)$, up to equivalent norms, provided that $\mathcal{Q}$ is an extension domain (see \Cref{appsec: theory-details}).

\subsection{Denoising and interpolation tests}
\label{subsec: denoising_interpolation_tests}

We describe here the denoising and interpolation tests based on downsampling reads and subsampling cells that we use to evaluate our method. Denote by $\mathbf{C}_0 \in \mathbb{N}_0^{m \times d}$ a pixels-by-genes counts matrix of unique molecular identifier (UMI) reads from a real spatial transcriptomics dataset. The entry $(\mathbf{C}_0)_{ij}$ counts the number of UMI reads of gene $j$ that were collected at pixel $q_i$. If the total number of reads $R_0$ is high, then $\mathbf{F}_0 = \mathrm{RowNormalize}(\mathbf{C}_0)$ is a good approximation to the underlying ground truth gene expression matrix $\mathbf{F}^\star$. In particular, the rows of $\mathbf{F}_0$, which are the observed gene expression vectors at pixels $q_1, \dots, q_m$, are
close to $F^\star(q_1), \dots, F^\star(q_m)$.
Given $\mathbf{C}_0$, we downsample it to $R< R_0$ reads, which means that we retain $R$ reads from $\mathbf{C}_0$ uniformly at random and discard the rest, resulting in a downsampled counts matrix $\mathbf{C}_{\mathrm{DS}}$. The matrix $\mathbf{Y}=\mathrm{RowNormalize}(\mathbf{C}_{\mathrm{DS}})$ has rows $Y_1,\dots,Y_m$ that are much noisier than those of $\mathbf{F}_0$ when $R \ll R_0$.

\paragraph{Denoising test}
Input $\mathbf{Y}$ into the denoising algorithm being studied, which outputs denoised gene expression vectors $\bar{F}(q_1),\dots, \bar{F}(q_m)$. Compare these to the rows of $\mathbf{F}_0$. 

\paragraph{Interpolation test}
Subsample $\mathbf{Y}$ to $\tilde{m}<m$ pixels by retaining $\tilde{m}$ pixels
$\{q_{i_1}, q_{i_2}, \dots, q_{i_{\tilde{m}}}\}$ selected uniformly at random and discarding the rest. The resulting subsampled matrix $\tilde{\mathbf{Y}}$ has the corresponding rows $Y_{i_1}, Y_{i_2},\dots, Y_{i_{\tilde{m}}}$. Input $\tilde{\mathbf{Y}}$ into a denoising algorithm that estimates gene expression vectors $\bar{F}(q_1),\dots, \bar{F}(q_m)$ at all $m$ pixels. Compare these to the rows of $\mathbf{F}_0$.

\subsection{Performance metrics}
\label{subsec: performance_metrics}
The performance metrics \emph{label transfer accuracy}, \emph{kNN overlap} and \emph{relative error} are described below. Recall the matrix $\mathbf{F}_0$ of originally observed gene expression vectors, and denote by $\bar{\mathbf{F}}$ a denoised gene expression matrix with rows $\bar{F}(q_1), \dots, \bar{F}(q_m)$.

\paragraph{Label transfer accuracy}
An important piece of information contained in gene expression is cell identity; cells of distinct types tend to occupy different regions in gene expression space. The metric of label transfer accuracy intends to quantify how well the denoising process preserves cell identities. At a high level, we use a $k$-nearest neighbour procedure to transfer cell type labels from the original gene expression vectors at the pixels (rows of $\mathbf{F}_{0}$) to the denoised gene expression vectors (rows of $\bar{\mathbf{F}}$), and record the fraction of the transferred labels that are correct. This fraction clearly lies between 0 (completely inaccurate) to 1 (completely accurate). The degree to which the labels can be transferred successfully indicates the quality of the denoised matrix. Details on the steps involved in this label transfer procedure (e.g. log-transform, PCA) are provided in \Cref{appsubsec: metrics_details}. 

We remark that label transfer accuracy captures the essential geometry of gene expression space, and treat it as our primary metric for evaluation and hyperparameter tuning.

\paragraph{kNN overlap}
A fundamental data structure that encodes essential information in a set of gene expression vectors is their $k$-nearest neighbour (kNN) graph (notice that similar graphs have already played a role in our denoising strategy, and also in the above metric of label transfer accuracy). Important clustering algorithms for cells in gene expression space also take the kNN graph as an input. 
This motivates our kNN overlap metric, which is closely related to the corresponding metric from \citet{ahlmann-eltze_ComparisonTransformationsSinglecell_2023}. It evaluates the denoised gene expression vectors (rows of $\bar{\mathbf{F}}$) by comparing their kNN graph against the kNN graph of the original gene expression vectors (rows of $\mathbf{F}_{0}$) post some transformations (e.g. log-transform, PCA) as detailed in \Cref{appsubsec: metrics_details}. From the resulting adjacency matrices $A_{0}, \bar{A} \in \{ 0,1 \}^{m\times m}$, compute the overlap as $\frac{1}{m} \langle A_{0}, \bar{A} \rangle_{F}$.
With the choice $k=50$ as in \citet{ahlmann-eltze_ComparisonTransformationsSinglecell_2023}, the kNN overlap lies between 0 (no overlap) to 50 (full overlap).

\paragraph{Relative error}
The relative error metric directly compares the original gene expression matrix $\mathbf{F}_0$ with the denoised matrix $\bar{\mathbf{F}}$ as
\begin{align*}
\mathrm{RelErr}(\mathbf{F}_0, \bar{\mathbf{F}}) := \frac{\fro{ \bar{\mathbf{F}} - \mathbf{F}_0}}{\fro{\mathbf{F}_0}}.
\end{align*}
This is a standard metric that is popular in the matrix denoising literature. Relative error can sometimes disagree with label transfer accuracy and kNN overlap, and we suspect this is because relative error fails to capture the geometry of gene expression space. 
{We expect there to be certain directions or substructures in gene expression space containing disproportionately more information on cell identities, but the relative error metric is agnostic to these.

While tuning hyperparameters of the denoising methods, we observed that the hyperparameter choices leading to the smallest relative errors often led to substantially worse label transfer accuracies and kNN overlaps. We also noticed that methods based on low-dimensional representations (e.g. low rank factorization) often did very well in terms of relative error, but struggled in the other metrics, likely because of misalignments between the low-dimensional representations involved and the biologically important directions or substructures.}

\subsection{Results}
\label{subsec: num_results}

Each snapshot of the MOSTA~\cite{chen_SpatiotemporalTranscriptomicAtlas_2022} dataset consists of a counts matrix $\mathbf{C}_0$ of unique molecular identifier (UMI) reads, together with $(x,y)$-coordinates for the pixels. We use the E9.5 snapshot, which, after some initial filtering, produces a matrix with $d=15717$ genes, $m=5503$ pixels, and $R_0\simeq 7.6 \times 10^{7}$ reads (i.e. $1.38 \times 10^4$ reads per pixel).
The gene expression vector at every pixel has already been assigned a cell type in MOSTA, through a procedure involving spatially-constrained-clustering. The result can be visualized as a usual image, with different colours representing the different cell types (see \Cref{fig:ground_truth_fig}).

\begin{figure}[t]
    \centering
    \includegraphics[width=0.4\textwidth]{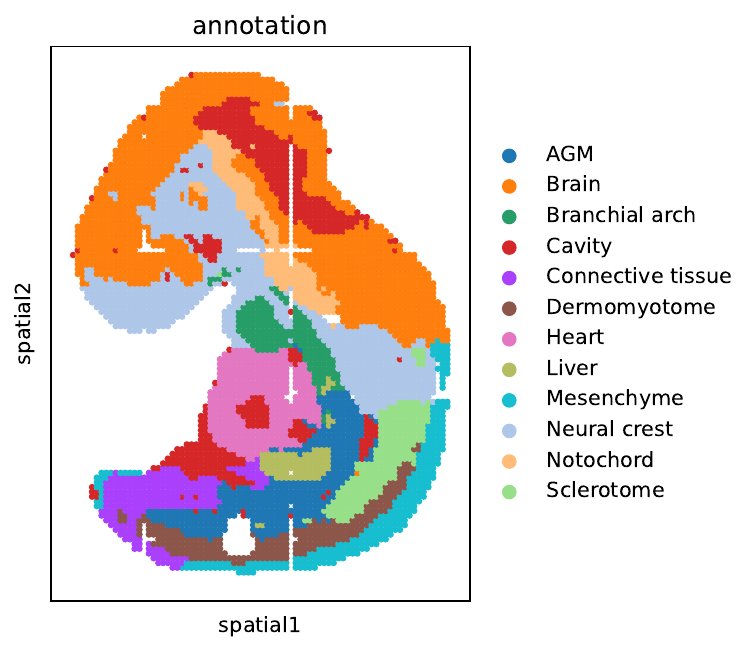}
    \caption{The original cell type assignment for the E9.5 snapshot of the MOSTA dataset. We treat these labels as ground truth for the label transfer accuracy metric.}
    \label{fig:ground_truth_fig}
\end{figure}

We first present results from the \textit{denoising test} on our method STARK, and compare against SPROD~\cite{wang_SprodDenoisingSpatially_2022}, GraphPCA~\cite{yang_GraphPCAFastInterpretable_2024} and STAGATE~\cite{dong_DecipheringSpatialDomains_2022}. Hyperparameter settings and transformation choices for these methods are made to \emph{optimize for label transfer accuracy}, as discussed further in \Cref{appsubsec: hyperparam-competing}. 
At various levels of downsampling as quantified by $R$, a noisy matrix $\mathbf{Y}$ is generated, which is subsequently denoised by the above methods. With 5 independent repeats for each experiment (to better account for the stochasticity in $\mathbf{Y}$), the averaged metrics for each of the methods are plotted against $R/m$ in \Cref{fig:denoising_test_comparisons}, together with bands of one standard deviation. STARK can be seen to perform particularly well in terms of label transfer accuracy and kNN overlap, followed closely by SPROD. Methods that use low-dimensional representations -- GraphPCA and STAGATE -- achieve very good relative errors, but in the process, fail to faithfully preserve biological information concerning cell identities and relative gene expression distances. 
\begin{figure}
    \centering
    \includegraphics[width=\textwidth, trim=0 15 0 0, clip]{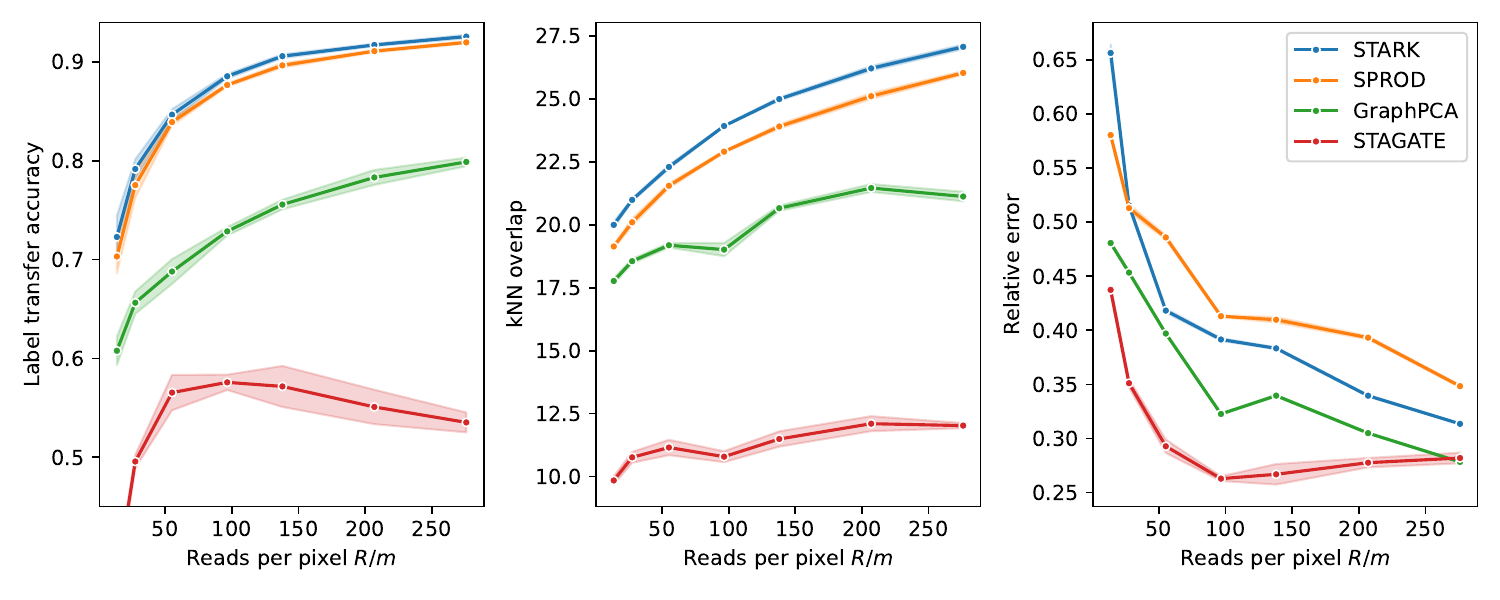}
    \caption{\textbf{Comparing denoising performance with SPROD, GraphPCA and STAGATE}. Noisy observations are generated by downsampling a counts matrix to various sequencing depths. The corresponding denoised images are evaluated in three metrics. Averaged results over 5 repeats are presented, together with bands of one standard deviation.}
    \label{fig:denoising_test_comparisons}
\end{figure}

Label transfer accuracies from the denoising test above can also be visualized conveniently; the cell type labels transferred onto noisy and denoised images are plotted in \Cref{fig:denoising_comparison_images}. 

\begin{figure}[!h]
    \centering
    \includegraphics[width=0.75\textwidth]{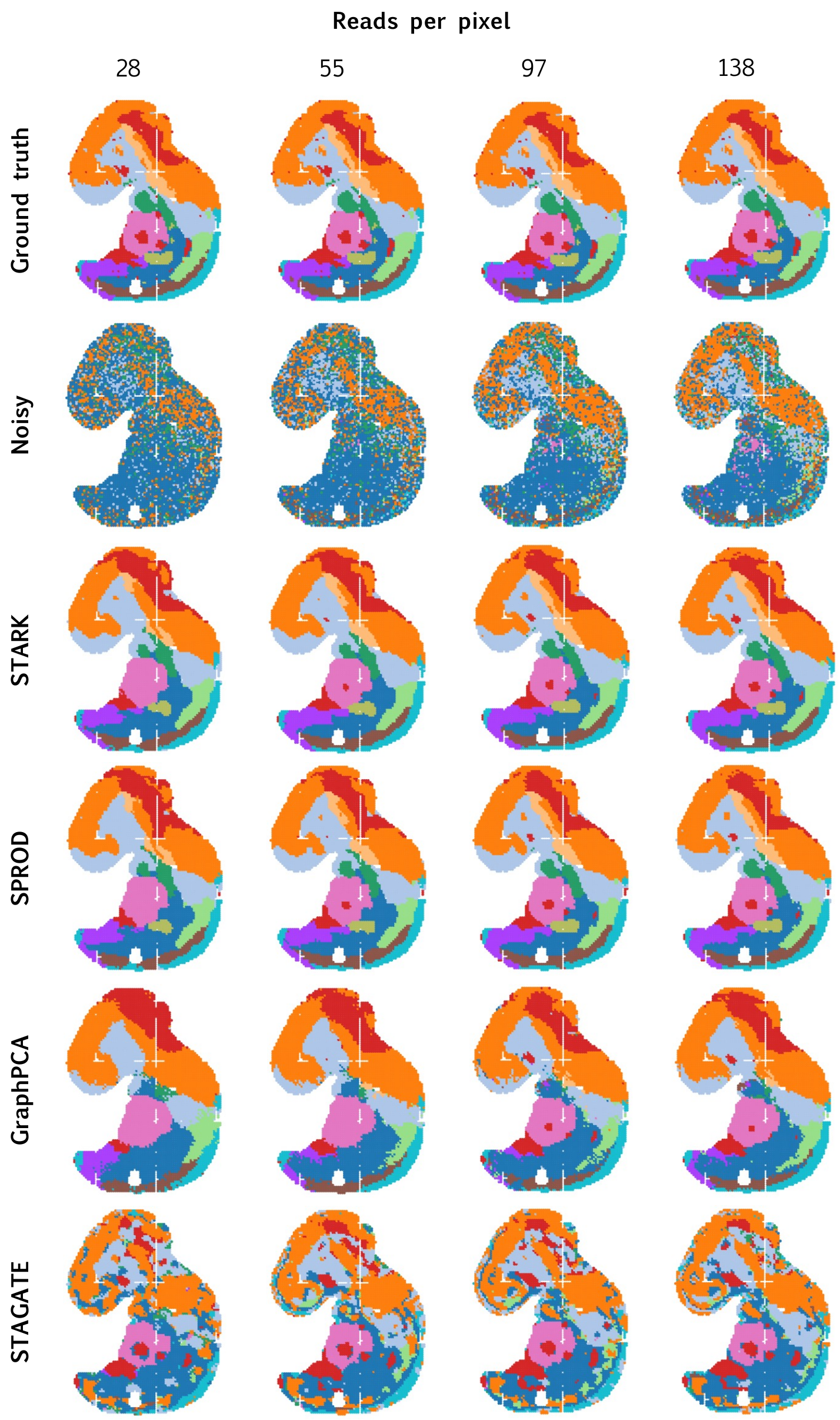}
    \caption{\textbf{Visually comparing denoising performance}. Colours represent cell type labels; see  \Cref{fig:ground_truth_fig}. Ground truth labels of the pixels are plotted in the first row. Labels transferred from the ground truth onto noisy and denoised images are plotted in the remaining rows.}
    \label{fig:denoising_comparison_images}
\end{figure}

Next to examine the role of adaptivity in our method, we also compare its performance against two other `variants'. The \textit{spatial} variant involves simply fixing the graph weights to 
\begin{equation*}
W_{\text{sp}}= \mathrm{RowNormalize}\left[ \exp \left( -|q_{k}-q_{i}|^2 /s_{2}^2 \right) \,\indicator_{\{ |q_{k}-q_{i}| \leq \tau \}} \right]_{i,k\in[m]}
\end{equation*}
and solving $F_{\text{sp}} = \argmin_{F \in \mathcal{H}^d} J(F, W_{\text{sp}})$. Notice that $F_{\text{sp}}$ simply coincides with the first iterate $F^1$ in \Cref{alg:iterative_update}. The graph here is constructed purely from spatial information, and is nonadaptive to the structure of the true image $F^\star$. On the other hand, the \textit{oracle} variant involves informing the graph weights directly by $F^\star$, which is unobserved in practice;
\begin{align*}
    W_{\text{oc}} = \mathrm{RowNormalize}\left[ \exp \left( -|\mathbf{F}^{\star}_{k\cdot}-\mathbf{F}^{\star}_{i\cdot}|^2 /s_{1}^2 \right)  \exp \left( -|q_{k}-q_{i}|^2 /s_{2}^2 \right) \,\indicator_{\{ |q_{k}-q_{i}| \leq \tau \}} \right]_{i,k\in[m]}.
\end{align*}
The oracle-denoised image is then $F_{\text{oc}}=\argmin_{F \in \mathcal{H}^d} J(F, W_{\text{oc}})$. Because it uses information provided by an `oracle', this variant is an idealized method, meant for benchmarking purposes.

We expect our incrementally adaptive method to denoise better than the spatial variant, and approach the performance of the oracle variant. This is indeed the case in denoising tests, as shown in \Cref{fig:spatial_oracle_comparisons}. The experimental settings are the same as in the previous test from \Cref{fig:denoising_test_comparisons}. Interestingly, the difference in performance between the variants is quite small, suggesting that nonadaptive regularization based on spatial smoothing can still denoise effectively. The adaptivity seems to be responsible for mainly finer improvements.

\begin{figure}[h]
    \centering
    \includegraphics[width=\textwidth, trim=0 15 0 0, clip]{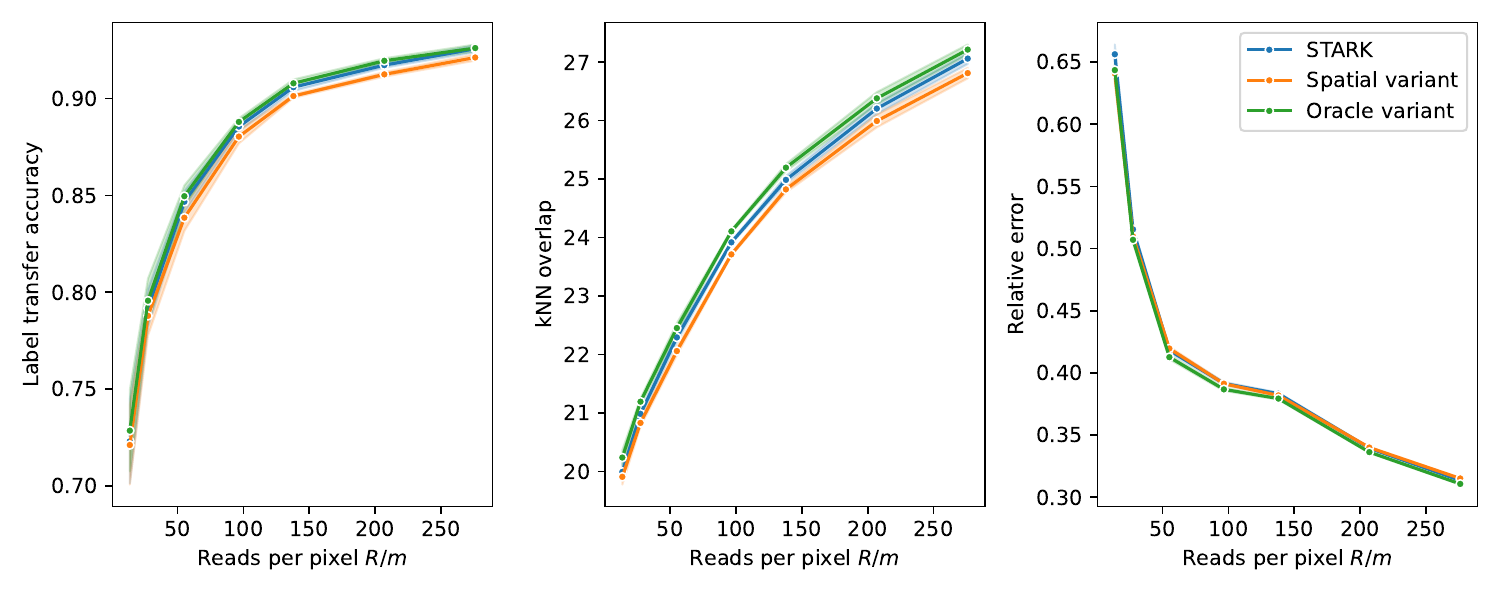}
    \caption{\textbf{Comparing denoising performance with spatial and oracle variants}. Noisy observations are generated by downsampling a counts matrix to various sequencing depths. The corresponding denoised images are evaluated in three metrics. Averaged results over 5 repeats are presented, together with bands of one standard deviation.}
    \label{fig:spatial_oracle_comparisons}
\end{figure}

Finally, we move on to results from the \textit{interpolation test} (see \Cref{subsec: denoising_interpolation_tests}) on STARK. By downsampling reads and discarding pixels, a noisy, subsampled matrix $\tilde{\mathbf{Y}}$ is generated. Subsampling is quantified by the ratio $\tilde{m}/m$ of observed pixels to total pixels. Downsampling, on the other hand, is indicated by the total number of reads $\tilde{R}$ at the observed $\tilde{m}$ pixels. 

For various levels of subsampling and downsampling, the gene expression vectors recovered by STARK at all $m$ pixels are evaluated in the same three metrics. With 5 independent repeats for each experiment (to better account for the stochasticity in $\tilde{\mathbf{Y}}$), the averaged metrics are plotted in \Cref{fig:cell-subsampling} together with bands of one standard deviation. 

The interpolation capabilities of STARK are demonstrated by decent performance even at low pixel subsampling ratios. The various curves in \Cref{fig:cell-subsampling} also inform on possible tradeoffs between sequencing depth (reads per pixel) and the number of observed pixels. 
The label transfer accuracies from this interpolation test can also be visualized as before. For a few pixel subsampling ratios, and with downsampling fixed to 97 reads per pixel, the cell type labels transferred onto the observed noisy pixels as well as the recovered images are plotted in \Cref{fig:pixel_subsampling_comparison_images}. Contrast these with the true labels from \Cref{fig:ground_truth_fig}.

\begin{figure}
    \centering
    \includegraphics[width=\textwidth, trim=0 15 0 0, clip]{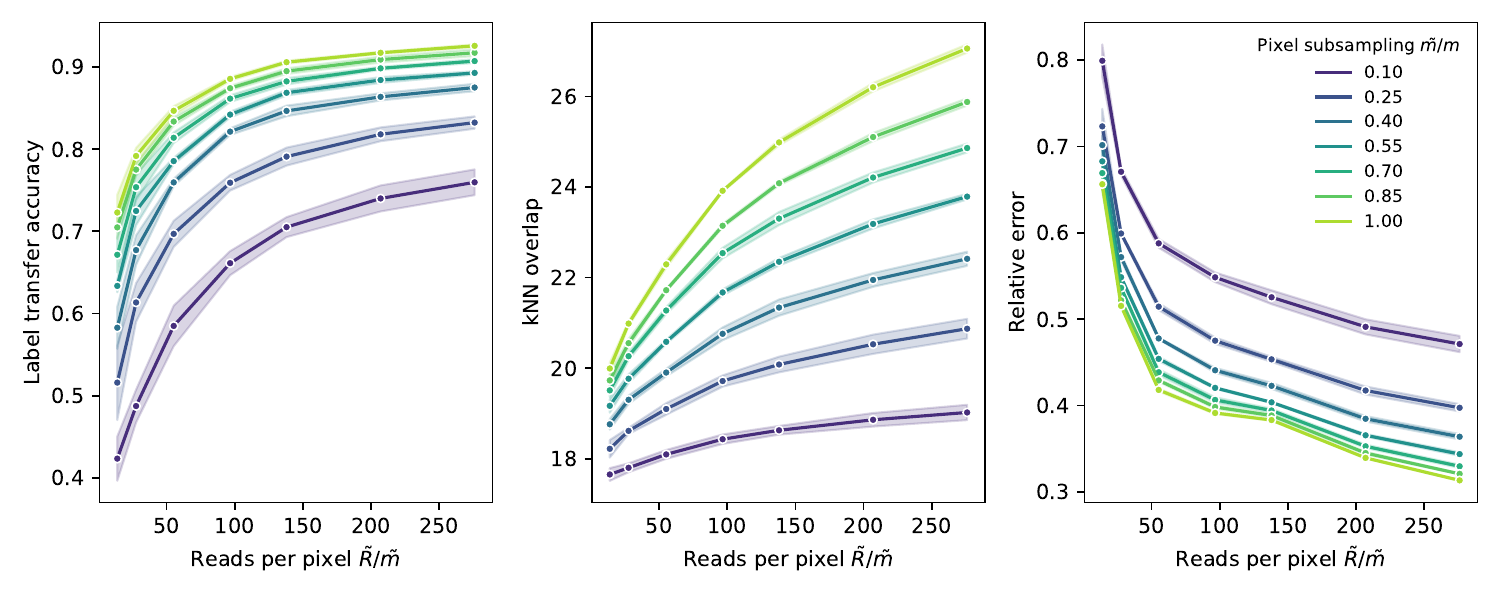}
    \caption{\textbf{Interpolation performance of STARK}. Noisy observations are generated by downsampling reads and subsampling pixels to varying extents. The corresponding images recovered by STARK are evaluated in three metrics. Averaged results over 5 repeats are presented, together with bands of one standard deviation.}
    \label{fig:cell-subsampling}
\end{figure}

\begin{figure}
    \centering
    \includegraphics[width=0.75\textwidth]{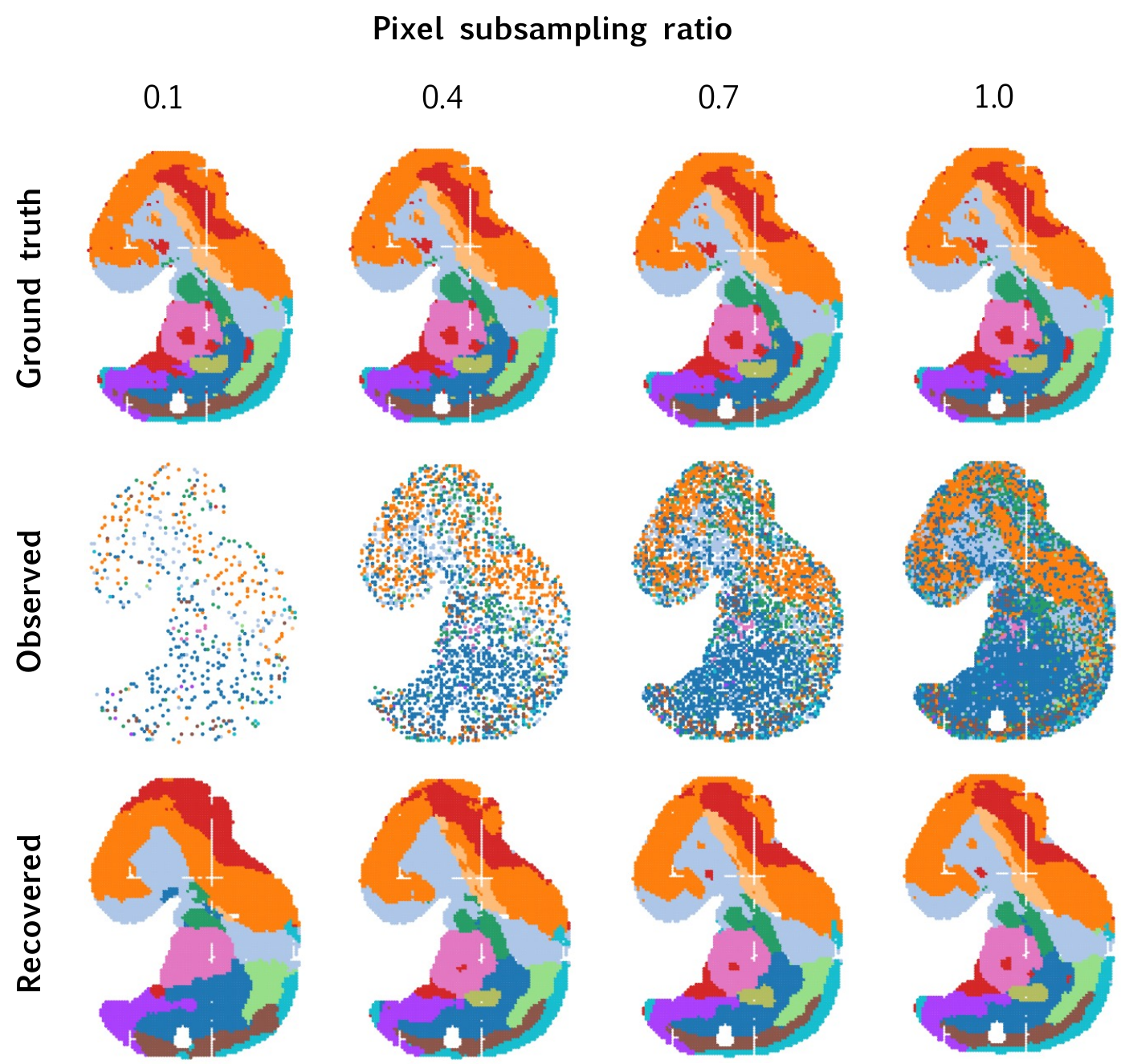}
    \caption{\textbf{Visually assessing interpolation performance}. Colours represent cell type labels; see  \Cref{fig:ground_truth_fig}. Ground truth labels of the pixels are plotted in the first row. Labels transferred from the ground truth onto noisy observations and recovered images are plotted in the remaining rows. Downsampling is fixed to 97 reads per pixel.}
    \label{fig:pixel_subsampling_comparison_images}
\end{figure}

\section{Proofs of key results}\label{sec: proofs}
We present the proofs of \Cref{loc:statement_(asymptotic_rate_with_respect_to_the_reads)} and \Cref{thm: strong_convergence} in this section. The remaining proofs are deferred to \Cref{appsec: deferred-proofs}.

\subsection{Proof of \Cref{loc:statement_(asymptotic_rate_with_respect_to_the_reads)}}

The following concentration bound \hyperlink{prf: binomial_concentration}{is proved} in \Cref{appsec: deferred-proofs}.

\begin{lemma}
\label{loc:bounds_for_rmin_via_binomial_and_poisson_mgfs.statement_(chernoff_bound_for_binomial_or_poisson_random_variables)}
Let $X \sim \mathrm{Bin}(R,p)$ or $X \sim \mathrm{Poi}(Rp)$. Then, for any $\tau \in (0, 1]$,
\begin{equation*}
\mathbb{P}\{ X \leq \tau Rp \} \leq \exp\left( -Rp\left( 1-\tau-\tau \log \tfrac{1}{\tau} \right) \right).
\end{equation*}
Picking $\tau=1 /4$ in particular gives $\mathbb{P}\left\{ X \leq \tfrac{1}{4}Rp \right\}\leq \exp\left( - \tfrac{1}{2}Rp \right)$.
\end{lemma}

\begin{proof}[\hypertarget{loc:statement_(asymptotic_rate_with_respect_to_the_reads).proof}Proof of \Cref{loc:statement_(asymptotic_rate_with_respect_to_the_reads)}]

Let $(\bar{F}, \bar{W}) \in \mathrm{Stat}(J)$ be any stationary point of $J$ over $\mathcal{H}^d \times \mathcal{C}_\tau$. We formulate an intermediate problem, and compare its solution $\check{F}$ to $\bar{F}$ and $F^\star$.

\paragraph{Step 1: Closeness of $F^\star$ to  $\check{F}$}
Define the auxiliary objective function
\begin{equation*}
\begin{aligned}
\check{J}(F,W):=  \|F - F^\star\|_{L^2_{m}}^2 + \lambda\|F\|_{\mathcal{H}^d}^2 + \omega \frac{1}{2m}\sum_{i,k=1}^m W_{ik}\euc{F(q_{k})-F(q_{i})}^2 + \omega E(W)  ,
\end{aligned}
\end{equation*}
and while fixing $W=W^\star\equiv \argmin_{W \in \mathcal{C}_\tau}J(F^\star, W)$, consider the minimizer
\begin{equation}
\label{eq:auxiliary_problem}
\check{F} = \argmin_{F \in \mathcal{H}^d}\,\check{J}(F, W^\star).
\end{equation}
Following an identical argument to the proof of \Cref{loc:statement_(representer_theorem)} (with $Y_{i}$ replaced by $F^\star(q_{i})$), the solution $\check{F}$ to the auxiliary problem \eqref{eq:auxiliary_problem} can be represented as $\check{F}(\cdot) = \frac{1}{\sqrt{ m }} \sum_{i=1}^m \mathcal{K}(q_{i}, \cdot)\check{\theta}_{i}$,
where the vectors $\check{\theta}_{1},\dots,\check{\theta}_{m} \in \mathbb{R}^d$ are the rows of the matrix
\begin{equation*}
\check{\boldsymbol{\theta}} = \frac{1}{\sqrt{ m }} \left( \mathbf{K}^2 + \lambda \mathbf{K} + \omega \mathbf{K}\overline{\mathbf{L}}_{W^\star}\mathbf{K} \right)^+ \mathbf{K}\mathbf{F}^\star.
\end{equation*}
We also have the basic inequality
\begin{equation*}
\|\check{F}-F^\star\|_{L^2_{m}}^2 \leq \lambda \left\{ \|F^\star\|_{\mathcal{H}^d}^2 - \|\check{F}\|_{\mathcal{H}^d}^2 \right\} + \omega \left\{ \frac{1}{m}\mathrm{tr}({\mathbf{F}^\star}^T \overline{\mathbf{L}}_{W^\star} \mathbf{F}^\star) -  \frac{1}{m}\mathrm{tr}(\check{\mathbf{F}}^T \overline{\mathbf{L}}_{W^\star} \check{\mathbf{F}}) \right\},
\end{equation*}
which is obtained by noting $\check{J}(\check{F}, W^\star) \leq \check{J}(F^\star, W^\star)$, rearranging, and using \eqref{eq: lap-trace}.
We drop the terms negative terms $-\|\check{F}\|_{\mathcal{H}^d}^2$ and $-\frac{1}{m}\mathrm{tr}(\check{\mathbf{F}}^T \overline{\mathbf{L}}_{W^\star} \check{\mathbf{F}})$ from the basic inequality and conclude 
\begin{equation*}
\|\check{F}-F^\star\|_{L^2_{m}}^2 = \mathcal{O}(R^{-1})
\end{equation*}
by the hyperparameter choice $\lambda=\mathcal{O}(R^{-1})$ and $\omega=\mathcal{O}(R^{-1})$.

\paragraph{Step 2: Closeness of $\bar{F}$ to $\check{F}$}
As $(\bar{F},\bar{W})$ is a stationary point of $J$ over $\mathcal{H}^d \times \mathcal{C}_\tau$ and $F \mapsto J(F,\bar W)$ is convex, it holds that $\bar{F} = \argmin_{F \in \mathcal{H}^d}\,J(F, \bar{W})$. Then, by \Cref{loc:statement_(representer_theorem)}, $\bar{F}(\cdot) = \frac{1}{\sqrt{ m }} \sum_{i=1}^m \mathcal{K}(q_{i}, \cdot)\overline{\theta}_{i}$,
where the vectors $\overline{\theta}_{1},\dots,\overline{\theta}_{m} \in \mathbb{R}^d$ are the rows of the matrix
\begin{equation*}
\begin{aligned}
\overline{\boldsymbol{\theta}} = \frac{1}{\sqrt{ m }}(\mathbf{K}^2+\lambda \mathbf{K}+\omega \mathbf{K}\overline{\mathbf{L}}_{\bar{W}}\mathbf{K})^+ \mathbf{K}\mathbf{Y}.
\end{aligned}
\end{equation*}
One can now express
\begin{align*}
\|\bar{F}-\check{F}\|_{L^2_{m}} &= \frac{1}{\sqrt{ m }} \fro{\bar{\mathbf{F}}-\check{\mathbf{F}}} = \frac{1}{\sqrt{ m }} \fro{\sqrt{ m }\mathbf{K}\overline{\boldsymbol{\theta}}- \sqrt{ m }\mathbf{K}\check{\boldsymbol{\theta}}} = \frac{1}{\sqrt{ m }} \fro{\mathbf{K}\bar{\mathbf{M}}\mathbf{K}\mathbf{Y} - \mathbf{K}\mathbf{M}^\star \mathbf{K} \mathbf{F}^\star},
\end{align*}
where the matrices $\bar{\mathbf{M}}:= \left( \mathbf{K}^2+\lambda \mathbf{K}+\omega \mathbf{K}\overline{\mathbf{L}}_{\bar{W}}\mathbf{K} \right)^+$ and $\mathbf{M}^\star:= \left( \mathbf{K}^2+\lambda \mathbf{K}+\omega \mathbf{K}\overline{\mathbf{L}}_{W^\star}\mathbf{K} \right)^+$. By the triangle inequality, split
\begin{equation}
\begin{aligned}
&\frac{1}{\sqrt{ m }} \fro{(\mathbf{K}\bar{\mathbf{M}}\mathbf{K})\mathbf{Y} - (\mathbf{K}\mathbf{M}^\star \mathbf{K}) \mathbf{F}^\star} \\
&\leq \frac{1}{\sqrt{ m }}\fro{(\mathbf{K}\bar{\mathbf{M}}\mathbf{K})(\mathbf{Y}-\mathbf{F}^\star)} + \frac{1}{\sqrt{ m }}\fro{\mathbf{K}(\bar{\mathbf{M}}-\mathbf{M}^\star)\mathbf{K}\mathbf{F}^\star}\\
&\leq \frac{1}{\sqrt{ m }} \op{\mathbf{K}}^2\op{\bar{\mathbf{M}}}\fro{\mathbf{Y}-\mathbf{F}^\star} + \frac{1}{\sqrt{ m }}\op{\mathbf{K}}^2\op{\bar{\mathbf{M}}-\mathbf{M}^\star}\fro{\mathbf{F}^\star}.
\end{aligned}\label{eq:triangle_split}
\end{equation}

We handle the second term via a perturbation analysis. Noting that
\begin{equation*}
\mathrm{rank} (\mathbf{K}^2+\lambda \mathbf{K}+\omega \mathbf{K}\overline{\mathbf{L}}_{\bar{W}}\mathbf{K}) = \mathrm{rank}(\mathbf{K}^2+\lambda \mathbf{K}+\omega \mathbf{K}\overline{\mathbf{L}}_{W^\star}\mathbf{K}) = \mathrm{rank}\,\mathbf{K},
\end{equation*}
a result by \citet[Theorem 4.1]{wedin_PerturbationTheoryPseudoinverses_1973} yields the bound
\begin{equation*}
\begin{aligned}
&\op{\bar{\mathbf{M}}-\mathbf{M}^\star} \leq 2\, \op{\bar{\mathbf{M}}} \op{\mathbf{M}^\star}\, \omega\op{\mathbf{K}}^2 \op{\overline{\mathbf{L}}_{\bar{W}}- \overline{\mathbf{L}}_{W^\star}}.
\end{aligned}
\end{equation*}
Letting $r=\mathrm{rank}(\mathbf{K})$, and denoting by $\Lambda_{r}(\cdot)$ the $r$-th largest eigenvalue of the input matrix,
\begin{equation*}
\begin{aligned}
\op{\bar{\mathbf{M}}} &= \frac{1}{\Lambda_{r}(\mathbf{K}^2 + \lambda \mathbf{K}+\omega \mathbf{K}\overline{\mathbf{L}}_{\bar{W}}\mathbf{K})} \leq \frac{1}{\Lambda_{r}(\mathbf{K}^2)} = \frac{1}{\Lambda_{r}(\mathbf{K})^2},
\end{aligned}
\end{equation*}
and similarly for $\op{\mathbf{M}^\star}$.
On the other hand, by the Gerschgorin circle theorem, we also have that $\op{\overline{\mathbf{L}}_{\bar{W}}}\leq 2+\max_{i\in[m]}|\mathrm{neb}(i)|$, where $|\mathrm{neb}(i)|$ denotes the number of neighbours of the $i$-th vertex in the graph $\mathcal{G}$.
These are just constant quantities as $R \to \infty$, and consequently,
\begin{equation*}
\op{\bar{\mathbf{M}}-\mathbf{M}^\star} = \mathcal{O}(\omega) = \mathcal{O}(R^{-1}).
\end{equation*}
We conclude this step by collecting the above inequalities as
\begin{align*}
    \|\bar{F}-\check{F}\|_{L^2_{m}}  &\leq \mathcal{O}(1) \fro{\mathbf{Y}-\mathbf{F}^\star} + \mathcal{O}(R^{-1}),
\end{align*}
where the $\mathcal{O}(1)$ factor is bounded with respect to $R$.

\paragraph{Step 3: Combining and taking expectation}
By the triangle inequality,
\begin{align*}
    \|\bar{F}-F^\star\|_{L^2_{m}} &\leq \|\check{F}-F^\star\|_{L^2_{m}} + \|\bar{F}-\check{F}\|_{L^2_{m}} \leq \mathcal{O}(R^{-1/2}) + \mathcal{O}(1) \fro{\mathbf{Y}-\mathbf{F}^\star} + \mathcal{O}(R^{-1}).
\end{align*}
Because this bound holds for any $(\bar{F}, \bar{W}) \in \mathrm{Stat}(J)$, we in fact have
\begin{align}
\label{eq: sup-bound-pre-exp}
    \sup \{ \|\bar{F} - F^\star\|_{L^2_{m}} \mid (\bar{F}, \bar{W}) \in \mathrm{Stat}(J) \} \leq \mathcal{O}(R^{-1/2}) + \mathcal{O}(1) \fro{\mathbf{Y}-\mathbf{F}^\star} + \mathcal{O}(R^{-1}),
\end{align}
and we can take expectation on both sides of \eqref{eq: sup-bound-pre-exp}.

What remains is to bound the expected noise $\mathbb{E}\fro{\mathbf{Y}-\mathbf{F}^\star}$ as $R \to \infty$. We start by separating ``low-reads" events $\mathscr{L}_{i}:=\left\{  R_{i} < u_{i}R /4  \right\}$ from ``high-reads" events $\mathscr{L}_{i}^c=\{ R_{i} \geq u_{i}R /4 \}$, for $i=1,\dots,m$, and write
\begin{align*}
\fro{\mathbf{Y}-\mathbf{F}^\star}^2 &=\sum_{i=1}^m \indicator_{\mathscr{L}_{i}}\euc{Y_{i}-F^\star(q_{i})}^2 + \sum_{i=1}^{m}\indicator_{\mathscr{L}_{i}^c}\euc{Y_{i}-F^\star(q_{i})}^2 \\
&\leq 2\sum_{i=1}^m \indicator_{\mathscr{L}_{i}} + \sum_{i=1}^{m}\indicator_{\mathscr{L}_{i}^c}\euc{Y_{i}-F^\star(q_{i})}^2.
\end{align*}
Conditional on the reads $\mathbf{R}=(R_{1},\dots R_{m})$, we have
\begin{align*}
\mathbb{E} \left[ \fro{\mathbf{Y}-\mathbf{F}^\star}^2 \mid \mathbf{R} \right]  &\leq 2\sum_{i=1}^m \indicator_{\mathscr{L}_{i}} + \sum_{i=1}^{m}\indicator_{\mathscr{L}_{i}^c}\mathbb{E}\left[ \euc{Y_{i}-F^\star(q_{i})}^2 \mid \mathbf{R} \right]  \\
&\leq 2\sum_{i=1}^m \indicator_{\mathscr{L}_{i}} + \sum_{i=1}^{m}\indicator_{\mathscr{L}_{i}^c}\, \frac{1}{R_{i}} \leq 2\sum_{i=1}^m \indicator_{\mathscr{L}_{i}} +  \frac{4m}{u_{\min}R}.
\end{align*}
Taking an unconditional expectation and using \Cref{loc:bounds_for_rmin_via_binomial_and_poisson_mgfs.statement_(chernoff_bound_for_binomial_or_poisson_random_variables)} gives
\begin{equation*}
\mathbb{E} \fro{\mathbf{Y}-\mathbf{F}^\star}^2 \leq 2m e^{ -u_{\min}R/2 } + \frac{4m}{u_{\min}R} = \mathcal{O}(R^{-1}).
\end{equation*}

To conclude the proof, note that the last bound combined with Jensen's inequality results in $\mathbb{E}\fro{\mathbf{Y}-\mathbf{F}^\star} = \mathcal{O}(R^{-1/2})$, which, together with~\eqref{eq: sup-bound-pre-exp}, yields the desired inequality.
\end{proof}

\subsection{Proof of \Cref{thm: strong_convergence}}
\label{prfsec: strong_convergence}

The convergence of the sequence $(F^t, W^t)_{t \in \mathbb{N}}$ generated by \eqref{eq: alternating_updates} is established through the corresponding finite-dimensional iterates $(\boldsymbol{\theta}^t, W^t)_{t \in \mathbb{N}}$ defined as
\begin{equation}
\label{eq: finite-dim-iterates}
\begin{aligned}
    \boldsymbol{\theta}^{t+1} = \argmin_{\boldsymbol{\theta}\in \mathbf{K}\mathbb{R}^{m \times d}}\,\mathbf{J}(\boldsymbol{\theta}, W^t), & \qquad
    W^{t+1} = \argmin_{W \in \mathcal{C}_{\tau}}\,\mathbf{J}(\boldsymbol{\theta}^{t+1}, W),
\end{aligned}\tag{U'}
\end{equation}
for the function $\mathbf{J}$ from \eqref{eq: bold_J}.
By \Cref{loc:statement_(representer_theorem)}, ${F}^t = F_{\boldsymbol{\theta}^t}$ where $F_{\boldsymbol{\theta}}(\cdot) := \tfrac{1}{\sqrt{ m }}\textstyle\sum_{1}^m \mathcal{K}(q_{i},\cdot)\theta_{i}$
for any $\boldsymbol{\theta}\in\mathbb{R}^{m\times d}$ with rows $\theta_{i}\in \mathbb{R}^d$.
The following lemma establishes boundedness of the iterates.
\begin{lemma}\label{lem: bounded_iterates}
The sequence $(\boldsymbol{\theta}^t, W^t )_{t \in \mathbb N}$ generated by the updates \eqref{eq: finite-dim-iterates} is bounded. Consequently, the sequence $(F^t)_{t \in \mathbb{N}}$ is bounded in $(\mathcal{H}^d, \|\cdot\|_{\mathcal{H}^d})$.
\end{lemma}
The \hyperlink{prf: bounded_iterates}{proof} is found in \Cref{appsec: deferred-proofs}. We will henceforth denote $\|(\boldsymbol{\theta}, W)\|^2 = \fro{\boldsymbol{\theta}}^2 + \fro{W}^2$.
The next lemma addresses limit points of $(\boldsymbol{\theta}^t, W^t)_{t \in \mathbb{N}}$; see \hyperlink{prf: limit_pts_critical}{the proof} in \Cref{appsec: deferred-proofs}.

\begin{lemma}\label{lem: limit_pts_critical}
    If $(\bar{\boldsymbol{\theta}}, \bar W)$ is a limit point of the sequence $(\boldsymbol{\theta}^{t},W^{t})_{t \in \mathbb N}$ generated by \eqref{eq: finite-dim-iterates}, then
\begin{equation}
\begin{aligned}\label{eq: fin-dim-stat}
\bar{\boldsymbol{\theta}} & \in \argmin_{\boldsymbol{\theta} \in \mathbb R^{m \times d}} \mathbf{J}(\boldsymbol{\theta}, \bar W),\quad & \bar W & \in \argmin_{W \in \mathcal C_\tau} \mathbf{J}(\bar{\boldsymbol{\theta}}, W). 
\end{aligned}
\end{equation}

\end{lemma}

\Cref{lem: bounded_iterates} guarantees the existence of limit points of $(\boldsymbol{\theta}^{t}, W^{t})_{t \in \mathbb N}$ and \Cref{lem: limit_pts_critical} ensures that all limit points necessarily satisfy \eqref{eq: fin-dim-stat}. It remains to be shown that this sequence converges to a unique limit point. We need yet another lemma before addressing such convergence.

\begin{lemma}
\label{res: iterates-bdd-away-simplex}
For all $t \in \mathbb{N}$ and $(i,k)\in \mathcal{E}_{\tau}$, we have $W^t_{ik} \geq \nu > 0$, where the constant $\nu$ does not depend on $t$ or $(i,k)$.    
\end{lemma}
See \hyperlink{prf: iterates-bdd-away-simplex}{the proof} in \Cref{appsec: deferred-proofs}.

In order to prove \Cref{thm: strong_convergence}, we invoke the Kurdyka-Łojasiewicz property.

\begin{definition}(Kurdyka-Łojasiewicz (KL) property) \label{def: KL-property}
A function $f:\mathbb R^n \to \mathbb R \cup \{+\infty\}$ is said to possess the Kurdyka-Łojasiewicz property at $\bar x \in \mathrm{dom} (\partial f)$ if
\begin{align*}
\frac{|f - f(\bar x)|^\mu}{\mathrm{dist}(0, \partial f)}
\end{align*}
is bounded in a neighbourhood around $\bar x$ for some $\mu \in [0,1)$.
\end{definition}
 Define the extended-valued function $\Jinf: \mathbb{R}^{m \times d} \times \mathbb{R}^{\mathcal{E}_\tau} \to \mathbb{R}\cup \{+\infty\}$, which coincides with $\mathbf{J}$ if $W \in \mathcal{C}_\tau$ and outputs $+\infty$ otherwise. 
Using the fact that $\mathbf{J}$ is real-analytic in an open neighbourhood of any point $(\bar{\boldsymbol \theta},\bar{W}) \in \mathbb{R}^{m \times d} \times \mathbb{R}_{++}^{\mathcal{E}_\tau}$, we immediately get the following lemma; see \citet{bolte2007lojasiewicz, xu_BlockCoordinateDescent_2013}.

\begin{lemma}\label{lem:KL-property-satisfied}
The function $\Jinf$ satisfies the KL-property at any $(\bar{\boldsymbol \theta} , \bar{W})\in \mathbb{R}^{m \times d} \times (\mathcal{C}_\tau \cap\mathbb{R}_{++}^{\mathcal{E}_\tau})$.
\end{lemma}

\begin{proof}[\hypertarget{prf: proof_of_strong_convergence}Proof of \Cref{thm: strong_convergence}]

Let $(\boldsymbol{\theta}^t, W^t)_{t \in \mathbb N}$ be the sequence generated by \eqref{eq: finite-dim-iterates}. Recall that, by \Cref{lem: bounded_iterates}, $(\boldsymbol{\theta}^t, W^t)_{t \in \mathbb N}$ is bounded and therefore possess at least one limit point. Let  $(\bar{\boldsymbol{\theta}},\bar W)$ denote one such limit point. By \Cref{lem: limit_pts_critical}, $(\bar{\boldsymbol{\theta}},\bar W)$ satisfies \eqref{eq: fin-dim-stat} so that it is a stationary point of $\Jinf$ and by \Cref{res: iterates-bdd-away-simplex}, $\bar W \in \mathcal C_\tau \cap \mathbb R^{\mathcal E_\tau}_{++}$. By \Cref{lem:KL-property-satisfied}, the KL-property is satisfied at $(\bar{\boldsymbol{\theta}},\bar W)$. In addition, as we demonstrate now, the iterates $(\boldsymbol{\theta}^t, W^t)_{t \in \mathbb N}$ satisfy the three conditions necessary to apply \cite[Theorem 2.9]{attouch2013convergence}. These are as follows:
\begin{enumerate}[leftmargin=12pt]
    \item \label{item: sufficient_decrease} \textit{Sufficient decrease}: there exists a fixed positive constant $a$ such that for every $t \in \mathbb N$
    \begin{align*}
    \Jinf(\boldsymbol{\theta}^{t},W^{t}) - \Jinf(\boldsymbol{\theta}^{t+1},W^{t+1}) \geq a \|(\boldsymbol{\theta}^{t+1}, W^{t+1}) - (\boldsymbol{\theta}^t, W^t)\|^2.
    \end{align*}

    \item \label{item: relative_error} \textit{Relative error condition}: there exists a fixed positive constant $b$ such that for every $t \in \mathbb N$, there exists $Z^{t+1} \in \partial \Jinf(\boldsymbol{\theta}^{t+1}, W^{t+1})$ such that
    \begin{equation}\label{eq: subdifferential_bound}
    \|Z^{t+1}\| \leq b \| (\boldsymbol{\theta}^{t+1}, W^{t+1}) - (\boldsymbol{\theta}^t, W^t) \|.
    \end{equation}
    
    \item \label{item: continuity_condition} \textit{Continuity}. There exists a subsequence $(\boldsymbol{\theta}^{t_j}, W^{t_j}) \to (\bar{\boldsymbol{\theta}},\bar W)$ with $\Jinf(\boldsymbol{\theta}^{t_j}, W^{t_j}) \to \Jinf(\bar{\boldsymbol{\theta}}, \bar W)$.
\end{enumerate}

\Cref{item: sufficient_decrease} holds by \eqref{eq: sufficient_decrease} in the proof of \Cref{lem: limit_pts_critical}. We now prove that \cref{item: relative_error} holds by constructing a particular choice of $Z^{t+1} \in \partial \Jinf(\boldsymbol{\theta}^{t+1},W^{t+1})$ and showing that \eqref{eq: subdifferential_bound} holds. Observe that the updates \eqref{eq: finite-dim-iterates} (with \Cref{loc:remark_(representer_theorem)}) translate to the following optimality conditions
\begin{align*}
    0 & = \nabla_{\boldsymbol{\theta}} \Jinf(\boldsymbol{\theta}^{t+1}, W^{t}), \quad 0 \in \partial_W \Jinf(\boldsymbol{\theta}^{t+1}, W^{t+1}).
\end{align*}
In addition, we note that $\partial \Jinf (\boldsymbol{\theta}^{t+1}, W^{t+1}) = \left\{ \nabla_{\boldsymbol{\theta}} \Jinf (\boldsymbol{\theta}^{t+1}, W^{t+1})\right\} \times \partial_W \Jinf(\boldsymbol{\theta}^{t+1}, W^{t+1})$. Therefore, if we let $Z^{t+1} := \left( \nabla_{\boldsymbol{\theta}} \Jinf (\boldsymbol{\theta}^{t+1}, W^{t+1}) - \nabla_{\boldsymbol{\theta}} \Jinf(\boldsymbol{\theta}^{t+1}, W^{t}), 0 \right)$,
then we have $Z^{t+1} \in \partial \Jinf (\boldsymbol{\theta}^{t+1}, W^{t+1})$. Next, we bound $\| Z^{t+1} \|$ by recalling the definition of $\mathbf{J}$ \eqref{eq: bold_J} as 
\begin{align*}
\| Z^{t+1} \| &= \| \nabla_{\boldsymbol{\theta}} \mathbf{J} (\boldsymbol{\theta}^{t+1}, W^{t+1}) - \nabla_{\boldsymbol{\theta}} \mathbf{J}(\boldsymbol{\theta}^{t+1}, W^{t})\|_F  = \| \mathbf{K} \overline{\mathbf{L}}_{W^{t+1}} \mathbf{K} \boldsymbol{\theta}^{t+1} - \mathbf{K} \overline{\mathbf{L}}_{W^t} \mathbf{K} \boldsymbol{\theta}^{t+1} \|_F \\
& \leq \| \mathbf{K} \|^2_{\text{op}} \| \overline{\mathbf{L}}_{W^{t+1}} - \overline{\mathbf{L}}_{W^t} \|_{\text{op}} \| \boldsymbol{\theta}^{t+1} \|_{F},
\end{align*}
where $\overline{\mathbf{L}}_{W} := \frac{1}{2}({I}_{m}+\mathrm{Diag}(W^T\mathbf{e})) - \frac{1}{2} \left( W+W^T \right)$. Thus,
\begin{align*}
\| \overline{\mathbf{L}}_{W^{t+1}} - \overline{\mathbf{L}}_{W^t} \|_{\text{op}} & = \frac{1}{2} \| (\mathrm{Diag}((W^{t+1} - W^t)^T\mathbf{e})) - \left( W^{t+1} - W^t \right) - \left( ( W^{t+1} )^T - (W^t)^T \right) \|_\text{op}
\\
& \leq \|W^{t+1} - W^t \|_\text{op} + \frac{\sqrt{m}}{2} \| W^{t+1} - W^t\|_{\text{op}} \leq \left( 1 + \frac{\sqrt m}{2} \right) \|W^{t+1} - W^t \|_{F}.
\end{align*}
Thus, \cref{item: relative_error} holds as long as $\| \boldsymbol{\theta}^{t+1} \|_{F}$ is bounded above, which it is, by \Cref{lem: bounded_iterates}. \Cref{item: continuity_condition} holds by the existence of convergent subsequences as a consequence of \Cref{lem: bounded_iterates}, and the continuity of $\Jinf$ on its effective domain.
By \cite[Theorem 2.9]{attouch2013convergence},
$(\boldsymbol{\theta}^t, W^t)$ converges to $(\bar{\boldsymbol{\theta}}, \bar W)$. 

Finally, by \Cref{loc:statement_(representer_theorem)}, each iterate $(F^t, W^t)$ defined according to \eqref{eq: alternating_updates} corresponds to a finite-dimensional iterate $(\boldsymbol{\theta}^t, W^t)$ satisfying \eqref{eq: finite-dim-iterates}. Note also that $\bar{\boldsymbol{\theta}}= \lim_t \boldsymbol{\theta}^t \in \mathbf{K}\mathbb{R}^{m \times d}$. Applying \Cref{loc:statement_(representer_theorem)} again, we conclude that as $(\boldsymbol{\theta}^t, W^t)$  converges to $(\bar{\boldsymbol{\theta}},\bar W)$ which is a stationary point of $\Jinf$, $(F^t, W^t)$ necessarily converges to a stationary point of $J$ over $\mathcal{H}^d \times \mathcal{C}_\tau$.
\end{proof}

\section{Discussion}
\label{sec: discussion}

We have addressed the problem of denoising and interpolating spatial transcriptomics images
via regression in a vector-valued RKHS. By additionally employing incrementally adaptive graph Laplacian regularization, our proposed method STARK is able to estimate the ground truth effectively at sequencing depths as low as 14 UMI reads per pixel. The method is supported with asymptotic statistical error bounds as the reads $R \uparrow \infty$, as well as a convergence result on the alternating minimization scheme. When evaluated in label transfer accuracy and kNN overlap, the denoising performance of STARK on a slice of the MOSTA dataset shows consistent improvement over the competing methods tested. 

The spatial transcriptomics denoising problem stipulates multiple directions for further theoretical investigation.
A key avenue is the finite-sample analysis of such a denoising method. Nonasymptotic error bounds would highlight the relationships and tradeoffs between the number of genes, pixels and reads. 
On the other hand, modeling gene expression images by functions in an RKHS opens up connections to Gaussian process theory. Describing distributions of such images by Gaussian processes could enable a Bayesian analysis of spatial trajectory inference, providing a theoretical base for the development of new algorithms.

Finally, in \cref{fig:denoising_test_comparisons}, we observe that STAGATE and GraphPCA, the two methods based in low-dimensional representations appear to perform the best in relative error, while simultaneously performing worse than STARK and SPROD in both label transfer accuracy and kNN overlap. This discrepancy across metrics calls into question the most appropriate choice of error measurement for testing performance. It also suggests that evaluating methods based solely on Frobenius norm relative error will create a bias towards low-dimensional representations that may perform worse in terms of classification results.

\section*{Acknowledgments}
N.G. is supported by the Wallenberg AI, Autonomous Systems and Software Program (WASP) funded by the Knut and Alice Wallenberg Foundation. A.W. is supported by the Burroughs Wellcome Fund. 
M.P.F. is supported by an NSERC Discovery Grant.
Y.P. is partially supported by an NSERC Discovery Grant and a Tier II Canada Research Chair in Data Science. G.S. is supported by an NSERC Discovery Grant, the
Michael Smith Health Research BC Scholar Award, and
United Therapeutics.

\bibliographystyle{myabbrvnat.bst}
\bibliography{modified-zotero-references,refs}

\newpage
\appendix

\section{Background information on reproducing kernel Hilbert spaces (RKHS)}
\label{appsec: theory-details}
Fix a ground set $\mathcal{Q}$. An RKHS $(\mathcal{H}, \langle \cdot,\cdot \rangle_{\mathcal{H}})$ is a Hilbert space of functions $f:\mathcal{Q}\to \mathbb{R}$ in which pointwise evaluation $f \mapsto f(q)$ is continuous for any $q \in \mathcal{Q}$.  Given the RKHS, there exists a unique positive semidefinite kernel function $\mathcal{K}:\mathcal{Q}\times \mathcal{Q}\to \mathbb{R}$ such that (1) the reproducing property holds i.e. $f(q)=\langle f, \mathcal{K}(q,\cdot) \rangle_{\mathcal{H}}$, and (2) the space of finite linear combinations
\begin{equation*}
\left\{ \textstyle\sum_{k=1}^K \beta_{k}\mathcal{K}(q_{k},\cdot) \mid K \in \mathbb{N},\, q_{1},\dots,q_{K} \in \mathcal{Q},\, \beta_{1},\dots,\beta_{K} \in \mathbb{R} \right\}
\end{equation*}
is dense in $\mathcal{H}$. Conversely, starting from a positive semidefinite kernel $\mathcal{K}$, there exists a unique RKHS satisfying (1) and (2).

\subsection{Linking Sobolev spaces and Matérn kernels}
\label{appsubsec: matern_sobolev}

Let $\mathcal{K}$ be a Matérn kernel defined as
\begin{equation}
\label{eq:matern_kernel}
\mathcal{K}(q,q') = \frac{2^{1-\nu}}{\Gamma(\nu)} \left( \sqrt{ \frac{2\nu}{l} } \euc{q-q'} \right)^{\nu} K_{\nu}\left( \sqrt{ \frac{2\nu}{l} } \euc{q-q'} \right)
\end{equation}
where $\nu>0$ is a smoothness parameter, $l>0$ is a length-scale, and $K_{\nu}$ is a modified Bessel function; see \citet[pg. 84]{rasmussen_GaussianProcessesMachine_2006}. Setting $\nu= 1 /2$ in \eqref{eq:matern_kernel} gives the exponential kernel
\begin{equation*}
\mathcal{K}(q,q')=\exp(- \euc{q-q'} /l).
\end{equation*}
Letting $\nu \to \infty$ on the other hand yields a Gaussian kernel.

If $\mathcal{Q}=\mathbb{R}^n$, then the RKHS generated by the Matérn kernel $\mathcal{K}$ coincides with the Sobolev space $H^s(\mathbb{R}^n)$ with $s=n /2 + \nu$ (possibly up to equivalent norms) \cite{rasmussen_GaussianProcessesMachine_2006, novak2018reproducing, porcu_MaternModelJourney_2023}. The same goes for vector-valued RKHSes and Sobolev spaces respectively.
Letting $n=2$ and $\nu=1 /2$, we conclude that the exponential kernel on $\mathbb{R}^2$ generates $H^{3/2}(\mathbb{R}^2)$ (and consequently, $H^{3/2}(\mathbb{R}^2; \mathbb{R}^d)$).

\subsection{Restricting to subdomains}

In \Cref{subsec: variational_formulation}, we consider $H^s(\mathcal{Q}; \mathbb{R}^d)$ for a subdomain $\mathcal{Q} \subseteq \mathbb{R}^2$, and assert that it is still generated by a Matérn kernel provided $s=1+\nu$ with $\nu >0$. Thus, it remains to verify that the link between Matérn kernels and Sobolev spaces is preserved under restriction.

In the case of an RKHS $\mathcal{H}$ over $\mathbb{R}^n$, restricting the kernel to values in a subdomain $\Omega\subseteq \mathbb{R}^n$ gives rise to an RKHS $\tilde{\mathcal{H}}$ over $\Omega$ \cite[Section 1.4]{berlinet_ReproducingKernelHilbert_2004} with norm 
\begin{align*}
    \Vert \tilde{f} \Vert_{\tilde{\mathcal{H}}}=\min_{f\in\mathcal{H}:f\upharpoonright_\Omega = \tilde{f}} \Vert f \Vert_{\mathcal{H}}.
\end{align*}
On the other hand, for domains $\Omega$ which are ``extension domains", it holds that there exists a constant $C>0$ depending on $\Omega,n,s$ such that 
\begin{align*}
    \Vert \tilde{f} \Vert_{H^s(\Omega)}\leq\min_{f\in\mathcal{H}:f\upharpoonright_\Omega = \tilde{f}} \Vert f \Vert_{H^s(\mathbb{R}^n)}\leq C \Vert \tilde{f} \Vert_{H^s(\Omega)}.
\end{align*}
In particular, \citet[Theorem 13.17]{leoni2017first} verifies that any $\Omega$ which is bounded and has Lipschitz boundary is an extension domain, when $s>0$ is an integer. The case of general $s>0$ is then obtained from \citet[p. 560]{leoni2017first}, who establishes the corresponding extension theorem for Besov spaces by interpolation (see also \cite{brezis2018gagliardo}), together with \cite[Theorem 6.7]{devore1993besov} which shows that the class of Besov spaces contains the class of fractional Sobolev spaces (under assumptions on $\Omega$ which are satisfied in this case). Altogether this shows that the RKHS norm $\Vert \cdot \Vert_{\tilde{\mathcal{H}}}$ is equivalent to the Sobolev norm $\Vert \cdot \Vert_{H^s(\Omega)}$ in our case where $\Omega=\mathcal{Q}$, with the same therefore also holding in the vector-valued case.

\section{Deferred proofs}
\label{appsec: deferred-proofs}

\begin{proof}[\hypertarget{loc:proof_(error_in_the_rkhs_norm)}Proof of \Cref{loc:statement_(error_in_the_rkhs_norm)}]

We start by orthogonally decomposing our RKHS into $M=M(q_{1},\dots,q_{m})$ and $M^\perp$, writing $\mathcal{H}^{d}=M \oplus M^{\perp}$. The ground truth $F^\star$ in particular can be decomposed as
\begin{equation*}
F^\star = \underbrace{ F^\star_{M} }_{ \in M } + \underbrace{ F^\star_{M^\perp} }_{ \in M^\perp },
\end{equation*}
where $F^\star_M$ is the orthogonal projection of $F^\star$ onto $M$, with $\|F^\star_{M}-F^\star\|_{\mathcal{H}^d} = \mathcal{A}(F^\star; \{ q_{1},\dots,q_{m} \})$. Note also that $F^\star_{M^\perp}(q_{i})= \sum_{j=1}^d\langle \mathcal{K}(q_{i}, \cdot)e_{j}\,,\, F^\star_{M^\perp} \rangle_{\mathcal{H}^d} \,e_{j}=0$, implying $F^\star(q_{i})=F^\star_{M}(q_{i})$ and $\mathbf{F}^\star = \mathbf{F}^\star_{M}$.

Let $(\bar{F}, \bar{W}) \in \mathrm{Stat}(J)$ be any stationary point of $J$ over $\mathcal{H}^d \times \mathcal{C}_\tau$. By the triangle inequality,
\begin{equation} \label{eq: corr-triangle}
\begin{aligned}
\|\bar{F}-F^\star\|_{\mathcal{H}^d} &\leq \|\bar{F}-F^\star_{M}\|_{\mathcal{H}^d} + \|F^\star_{M}-F^\star\|_{\mathcal{H}^d} \\
&= \|\bar{F}-F^\star_{M}\|_{\mathcal{H}^d} + \mathcal{A}(F^\star; \{ q_{1},\dots,q_{m} \}).
\end{aligned}
\end{equation}
Since $\bar{F}-F^\star_{M} \in M$, one can express
\begin{equation}\label{eq: err-in-subspace}
\begin{aligned}
\|\bar{F}-F^\star_{M}\|_{\mathcal{H}^d}^2 &= \frac{1}{m} \mathrm{tr}\left[ (\bar{\mathbf{F}}-\mathbf{F}^\star_{M})^T \mathbf{K}^+ (\bar{\mathbf{F}}-\mathbf{F}^\star_{M}) \right] \\
&= \frac{1}{m} \mathrm{tr}\left[ (\bar{\mathbf{F}}-\mathbf{F}^\star)^T \mathbf{K}^+ (\bar{\mathbf{F}}-\mathbf{F}^\star) \right] \\
&= \frac{1}{m} \langle (\bar{\mathbf{F}}-\mathbf{F}^\star), \mathbf{K}^+ (\bar{\mathbf{F}}-\mathbf{F}^\star) \rangle_{F} \\
&\leq \frac{1}{m} \op{\mathbf{K}^+} \fro{(\bar{\mathbf{F}}-\mathbf{F}^\star)}^2 = \op{\mathbf{K}^+} \|\bar{F}-F^\star\|^2_{L^2_{m}}.
\end{aligned}
\end{equation}
Putting \eqref{eq: corr-triangle} and \eqref{eq: err-in-subspace} together gives
$$
\|\bar{F}-F^\star\|_{\mathcal{H}^d} \leq \mathcal{A}(F^\star; \{ q_{1},\dots,q_{m} \}) + \op{\mathbf{K}^+}^{1/2}\,\|\bar{F}-F^\star\|_{L^2_{m}}.
$$
Taking a supremum over $(\bar{F}, \bar{W}) \in \mathrm{Stat}(J)$ followed by an expectation, and using \Cref{loc:statement_(asymptotic_rate_with_respect_to_the_reads)} gives the desired inequality.
\end{proof}

\begin{proof}[\hypertarget{loc:proof_(representer_theorem)}Proof of \Cref{loc:statement_(representer_theorem)}]

Recall the notation
\begin{align*}
    F_{\boldsymbol{\theta}}(\cdot) = \frac{1}{\sqrt{ m }}\sum_{i=1}^m \mathcal{K}(q_{i},\cdot)\theta_{i}
\end{align*}
for any $\boldsymbol{\theta}\in\mathbb{R}^{m\times d}$ with rows $\theta_{1},\dots,\theta_{m} \in \mathbb{R}^d$. We proceed via the following steps.

\paragraph{Reducing to finite dimensions}

Consider the finite-dimensional subspace $M=\{ F_{\boldsymbol{\theta}} \mid \boldsymbol{\theta} \in \mathbb{R}^{m \times d} \}$ of $\mathcal{H}^d$, and recall the orthogonal decomposition $\mathcal{H}^{d}=M \oplus M^{\perp}$. For any $F \in \mathcal{H}^d$, writing
\begin{equation*}
F= \underbrace{ F_{M} }_{ \in M } + \underbrace{ F_{M^\perp} }_{ \in M^\perp },
\end{equation*}
and noting $F_{M^\perp}(q_{i})= \sum_{j=1}^d\langle \mathcal{K}(q_{i}, \cdot)e_{j}\,,\, F_{M^\perp} \rangle_{\mathcal{H}^d} \,e_{j}=0$ gives, by the Pythagorean theorem, that 
\begin{equation*}
\begin{aligned}
J(F, W) &= \frac{1}{m}\sum_{i=1}^{m} \euc{Y_{i}-F(q_{i})}^2 + \lambda\|F\|_{\mathcal{H}^d}^2 + \omega\left( \frac{1}{2m}\sum_{ik}W_{ik}\euc{F(q_{i})-F(q_{k})}^2 + E(W)  \right) \\
&= \frac{1}{m}\sum_{i=1}^{m} \euc{Y_{i}-F_{M}(q_{i})}^2 + \lambda\|F\|_{\mathcal{H}^d}^2 + \omega\left( \frac{1}{2m}\sum_{ik}W_{ik}\euc{F_{M}(q_{i})-F_{M}(q_{k})}^2 + E(W)  \right) \\
&= J(F_{M}, W) + \lambda\|F_{M^\perp}\|_{\mathcal{H}^d}^2.
\end{aligned}
\end{equation*}
In particular, $J(F, W) \geq J(F_{M},W)$ for all $F \in \mathcal{H}^d$, where equality holds if and only if $F_{M^\perp}=0$. As a consequence, any minimizer of $J(\cdot,W)$ over $\mathcal{H}^d$ must belong to $M$ so that
\begin{equation*}
\argmin_{F \in \mathcal{H}^d}\, J(F, W) = \argmin_{F \in M}\,J(F, W).
\end{equation*}

\paragraph{Handling the finite-dimensional problem}

Because the finite-dimensional subspace $M$ is parametrized by $\boldsymbol{\theta}\in \mathbb{R}^{m\times d}$, we know that $F_{\hat{\boldsymbol{\theta}}} \in \argmin_{F \in M}\,J(F,W)$ if and only if $\hat{\boldsymbol{\theta}} \in \argmin_{\boldsymbol{\theta} \in \mathbb{R}^{m\times d}}\, J(F_{\boldsymbol{\theta}}, W)$.
By strong convexity, any minimizer of $J(\cdot, W)$ is unique. Next, observe that
\begin{equation*}
\begin{aligned}
J(F_{\boldsymbol{\theta}},W) &= \frac{1}{m}\sum_{i=1}^{m} \euc{Y_{i}-F_{\boldsymbol{\theta}}(q_{i})}^2 + \lambda\|F_{\boldsymbol{\theta}}\|_{\mathcal{H}^d}^2 + \omega\left( \frac{1}{2m}\sum_{ik}W_{ik}\euc{F_{\boldsymbol{\theta}}(q_{i})-F_{\boldsymbol{\theta}}(q_{k})}^2 + E(W)  \right)\\
&= \frac{1}{m} \fro{\mathbf{Y}-\sqrt{ m }\mathbf{K}\boldsymbol{\theta}}^2 + \lambda\mathrm{tr}(\boldsymbol{\theta}^T\mathbf{K}\boldsymbol{\theta}) + \omega\left( \frac{1}{m}\mathrm{tr}\left[ (\sqrt{ m }\mathbf{K}\boldsymbol{\theta})^T \,\overline{\mathbf{L}}_W\,(\sqrt{ m }\mathbf{K}\boldsymbol{\theta}) \right] + E(W) \right) \\
&= \frac{1}{m} \fro{\mathbf{Y}-\sqrt{ m }\mathbf{K}\boldsymbol{\theta}}^2 + \lambda\mathrm{tr}(\boldsymbol{\theta}^T\mathbf{K}\boldsymbol{\theta}) + \omega\left( \mathrm{tr}(\boldsymbol{\theta}^T \mathbf{K} \overline{\mathbf{L}}_{W}\mathbf{K}\boldsymbol{\theta}) + E(W) \right) \\
&= \mathbf{J}(\boldsymbol{\theta}, W).
\end{aligned}
\end{equation*}
We conclude this step by showing that restricting $\boldsymbol{\theta}$ to the subspace $\mathbf{K}\mathbb{R}^{m\times d}$ is possible, and additionally yields uniqueness. Firstly, if $\boldsymbol{\alpha} \in (\mathbf{K}\mathbb{R}^{m\times d})^\perp$, then $\langle \mathbf{K}\boldsymbol{\alpha}, \mathbf{X} \rangle_{F} = \mathrm{tr}(\boldsymbol{\alpha}^T\mathbf{K}\mathbf{X}) = \langle \boldsymbol{\alpha}, \mathbf{K}\mathbf{X} \rangle_{F}=0$ for all $\mathbf{X}\in \mathbb{R}^{m\times d}$ implying that $\mathbf{K}\boldsymbol{\alpha}=0$, and consequently, $\mathbf{J}(\boldsymbol{\theta}+\boldsymbol{\alpha}, W)=\mathbf{J}(\boldsymbol{\theta},W)$. As a result, we can always replace $\boldsymbol{\theta}$ by its (Euclidean) projection onto $\mathbf{K}\mathbb{R}^{m\times d}$. On the other hand, suppose $\hat{\boldsymbol{\theta}}, \hat{\boldsymbol{\theta}}' \in \mathbf{K}\mathbb{R}^{m\times d}$  are minimizers of $\mathbf{J}(\cdot,W)$. Then $F_{\hat{\boldsymbol{\theta}}}$ and $F_{\hat{\boldsymbol{\theta}}'}$ are minimizers of $J(\cdot, W)$ over $\mathcal{H}^d$, and hence, must be equal by uniqueness. As a result,
\begin{equation*}
\begin{aligned}
0=\|F_{\hat{\boldsymbol{\theta}}}-F_{\hat{\boldsymbol{\theta}}'}\|_{\mathcal{H}^{d}}^2 &= \mathrm{tr}\left[ (\hat{\boldsymbol{\theta}}'- \hat{\boldsymbol{\theta}})^T \mathbf{K} (\hat{\boldsymbol{\theta}}'- \hat{\boldsymbol{\theta}})\right]
\end{aligned}
\end{equation*}
implying that $(\hat{\boldsymbol{\theta}}'- \hat{\boldsymbol{\theta}}) \in \mathbf{K}\mathbb{R}^{m\times d} \cap (\mathbf{K}\mathbb{R}^{m\times d})^\perp = \{ 0 \}$.

\paragraph{Obtaining the explicit expression}

Noting the gradient
\begin{equation*}
\begin{aligned}
\nabla_{\boldsymbol{\theta}}\mathbf{J}(\boldsymbol{\theta}, W) &= \frac{2}{m} \sqrt{ m }\mathbf{K}(\sqrt{ m }\mathbf{K}\boldsymbol{\theta}-\mathbf{Y}) + 2\lambda \mathbf{K}\boldsymbol{\theta} + 2\omega \mathbf{K}\overline{\mathbf{L}}_{W}\mathbf{K}\boldsymbol{\theta},
\end{aligned}
\end{equation*}
we have that $\boldsymbol{\theta}$ minimizes $\mathbf{J}(\cdot, W)$ over $\mathbb{R}^{m\times d}$ if and only if it solves
\begin{equation}\label{eq: theta_min_system}
(\mathbf{K}^2+\lambda \mathbf{K} + \omega \mathbf{K}\overline{\mathbf{L}}_{W}\mathbf{K}) \boldsymbol{\theta} = \frac{1}{\sqrt{ m }}\mathbf{K}\mathbf{Y}.
\end{equation}
Notice that the matrices $\mathbf{K}$, $(\mathbf{K}^2+\lambda \mathbf{K}+\omega \mathbf{K}\overline{\mathbf{L}}_{W}\mathbf{K})$ and $(\mathbf{K}^2+\lambda \mathbf{K}+\omega \mathbf{K}\overline{\mathbf{L}}_{W}\mathbf{K})^+$ share the same null space and range (due to symmetry and positive semidefiniteness). As a result,
\begin{align}\label{eq: theta_hat_min}
\hat{\boldsymbol{\theta}} = \frac{1}{\sqrt{ m }}(\mathbf{K}^2+\lambda \mathbf{K}+\omega \mathbf{K}\overline{\mathbf{L}}_{W}\mathbf{K})^+ \mathbf{K}\mathbf{Y}
\end{align}
is a solution, and additionally, belongs to $\mathbf{K}\mathbb{R}^{m\times d}$.
\end{proof}

\begin{proof}[\hypertarget{loc:proof_of_W_min_solution}Proof of \Cref{prop: W_min_solution}]

Ignoring terms that depend only on $F$, we can write
\begin{equation*}
\begin{aligned}
\argmin_{W \in \mathcal{C}_{\tau}}\, J(F, W) = \argmin_{W \in\mathcal{C}_{\tau}}\,\tilde{J}_{F}(W),
\end{aligned}
\end{equation*}
where
\begin{equation*}
\begin{aligned}
\tilde{J}_{F}(W) = \sum_{(i,k) \in \mathcal{E}_{\tau}} \left\{ W_{ik}\euc{F(q_{i})-F(q_{k})}^2 + s_{1}^2 \,W_{ik}\left( \log W_{ik}-1 \right) + \frac{s_{1}^2}{s_{2}^2}\,W_{ik} \euc{q_{i}-q_{k}}^2 \right\}.
\end{aligned}
\end{equation*}
First, we show that $\tilde{J}_{F}$ is strongly convex over the set
\begin{equation*}
\mathcal{B}_\tau:= \left\{ W \in \mathbb{R}^{\mathcal{E}_{\tau}} \mid 0 \leq W_{ik} \leq 1,\,\forall (i,k)\in \mathcal{E}_{\tau} \right\}.
\end{equation*}
The function $\tilde{J}_{F}$ is differentiable on $\mathrm{int}(\mathcal{B}_\tau)$, and at all $W \in \mathrm{int}(\mathcal{B}_\tau)$, the Hessian $H \in \mathbb R^{\mathcal{E}_\tau \times \mathcal{E}_\tau}$ is a diagonal matrix with entries $(H)_{ll}$, where $l = (i,k) \in \mathcal{E}_\tau$, given by
\begin{equation*}
(H)_{ll} = \frac{\partial^2 \tilde{J}_{F}(W)}{\partial W_{ik}^2 } = \frac{1}{W_{ik}} \geq 1.
\end{equation*}
In other words, $H \succeq I_{|\mathcal{E}_{\tau}|}$, and thus $\tilde{J}_{F}$ is 1-strongly convex on $\mathrm{int}(\mathcal B_\tau)$. By continuity of $\tilde{J}_{F}$ on $\mathcal{B}_\tau$, this strong convexity extends to all of $\mathcal{B}_\tau$. 

In order to fully account for the constraint $W \in \mathcal C_{\tau}$, we define the following affine set
\begin{equation}
\begin{aligned}
\label{eq:a_affine_set}
\mathcal A_\tau := \left\{ W \in \mathbb R^{\mathcal E_\tau} \mid  \textstyle\sum_{k \colon (i,k) \in \mathcal E_\tau } W_{ik} = 1\,\, \forall i \in [m] \right\}
\end{aligned}
\end{equation}
and the associated convex function
\begin{equation}\label{eq:a_indicator}
\delta_{\mathcal A_\tau}(W):= \begin{cases} 
0 & \text{ if } W \in \mathcal{A_\tau}; \\
+ \infty & \text{ otherwise.}
\end{cases}
\end{equation}
Then, $\tilde{J}_{F} + \delta_{\mathcal A_\tau}$ is $1$-strongly convex on $\mathcal{B}_\tau$, and in particular, $\tilde{J}_{F}$ is $1$-strongly convex on $\mathcal C_\tau=\mathcal{A_\tau}\cap\mathcal{B}_\tau$. Consequently, $\mathcal{C}_{\tau} \ni W \mapsto J(F,W)= (\omega /2m) \tilde{J_{F}}(W) + (\text{terms depending only on }F$) is $(\omega /2m)$-strongly convex.

Now, we show that $\hat{W}$ as defined in the statement of \Cref{prop: W_min_solution} satisfies the (necessary and sufficient) conditions for $\hat{W} \in \argmin_{W \in \mathcal{C}_{\tau}} \,\tilde{J}_{F}(W)$, namely
\begin{equation*}
\begin{aligned}
\hat{W} &\in \mathcal{C}_{\tau} \\
-\nabla_{W}\tilde{J}_{F}(\hat{W}) &\in N_{\mathcal{C}_{\tau}}(\hat{W}),
\end{aligned}
\end{equation*}
where $N_{\mathcal{C}_{\tau}}(\hat{W})$ denotes the normal cone to $\mathcal{C}_{\tau}$ at $\hat{W}$. The feasibility condition $\hat{W} \in \mathcal{C}_{\tau}$ is clear because $\tilde{W}_{ik}$ is strictly positive for $(i,k) \in \mathcal{E}_{\tau}$, and row-normalizing yields row-sums of 1. For optimality, we first note that $\hat{W} \in \mathrm{int}(\mathcal{B})$, where $\tilde{J}_{F}$ is differentiable with gradient
\begin{equation*}
[\nabla_{W}\tilde{J}_{F}(\hat{W})]_{ik} = \euc{F(q_{i})-F(q_{k})}^2 + s_{1}^2 \log \hat{W}_{ik} + \frac{s_{1}^2}{s_{2}^2}\euc{q_{i}-q_{k}}^2, \quad (i,k)\in \mathcal{E}_{\tau}.
\end{equation*}
On the other hand, the normal cone at $\hat{W}$ can be described as
\begin{equation*}
N_{\mathcal{C}_{\tau}}(\hat{W}) = \left\{ M \in \mathbb{R}^{\mathcal{E}_{\tau}} \mid M_{ik}=\lambda_{i},\,\,\forall(i,k) \in \mathcal{E}_{\tau}\,\text{ for some } \lambda_{1},\dots,\lambda_{m} \in \mathbb{R} \right\},
\end{equation*}
because none of the inequality constraints are active at $\hat{W}$ (see \cite[Section A.5.2]{hiriart-urruty_FundamentalsConvexAnalysis_2001}). Finally, recalling \eqref{eq: W_update} and that $\hat{W}_{ik} = \tilde{W}_{ik}/\sum_{\ell}\tilde{W}_{i\ell}$, we verify
\begin{align*}
& \left[ \nabla_{W} \tilde{J}_{F}(\hat{W}) \right]_{ik}  \\
&= \euc{F(q_{i})-F(q_{k})}^2 + s_{1}^2 \log \hat{W}_{ik} + \frac{s_{1}^2}{s_{2}^2}\euc{q_{i}-q_{k}}^2 \\
&= -s_{1}^2 \log \left( \textstyle \sum_{\ell}\tilde{W}_{i\ell} \right) =: -\lambda_{i}.
\end{align*}
\end{proof}

\begin{proof}[\hypertarget{prf: bounding-non-stationarity}Proof of \Cref{res: bounding-non-stationarity}]
The first statement is immediate by using the first-order optimality conditions for $W^t = \argmin_{W \in \mathcal{C}_{\tau}} J(F^t, W)$. Proving the second statement requires computing gradients in the RKHS $\mathcal{H}^d$. Define the evaluation maps $T_{i}:\mathcal{H}^d \to \mathbb{R}^d$ by $T_{i}F := F(q_{i})$. These are bounded linear maps from $\mathcal{H}^d$ to $\mathbb{R}^d$, with adjoint
\begin{equation*}
T_{i}^*X = \mathcal{K}(q_{i},\cdot)X, \qquad X \in \mathbb{R}^d.
\end{equation*}
We compute 
\begin{equation*}
\nabla_{F}\euc{T_{i}F - Y_{i}}^2 = 2\, T_{i}^*(T_{i}F - Y_{i}),
\end{equation*}
and similarly,
\begin{equation*}
\nabla_{F} \euc{T_{i}F - T_{k}F}^2 = 2\,(T_{i}-T_{k})^*(T_{i}-T_{k})F.
\end{equation*}
One also has the simple identity $\nabla_{F}\|F\|_{\mathcal{H}^d}^2 = 2F$. Putting these together gives
\begin{equation}
\begin{aligned}
\label{eq:rkhs_grad}
\nabla_{F}J(F, W^t) = A^tF - \frac{2}{m}\sum_{i=1}^mT_{i}^*Y_{i}
\end{aligned}
\end{equation}
where the bounded linear operator
\begin{equation*}
A^t := \frac{2}{m}\sum_{i=1}^m T_{i}^*T_{i} + 2\lambda + \frac{\omega}{m}\sum_{(i,k)\in \mathcal{E}_{\tau}}W_{ik}^t(T_{i}-T_{k})^*(T_{i}-T_{k}).
\end{equation*}
Because $0 \leq W_{ik}^t \leq 1$ for all $(i,k)$ and $t$, the sequence $(A^t)_{t \in \mathbb{N}}$ is uniformly bounded in operator norm i.e.
\begin{equation*}
\|A^t\|_{\mathcal{H}^d \to \mathcal{H}^d} \leq B,
\end{equation*}
for a constant $B$ that does not depend on $t$. 

To conclude the proof, recall that $F^{t+1}= \argmin_{F \in \mathcal{H}^d} J(F, W^t)$ and use first-order optimality conditions together with \eqref{eq:rkhs_grad} to obtain
\begin{equation*}
A^tF^{t+1} = \frac{2}{m}\sum_{i=1}^m T_{i}^*Y_{i}.
\end{equation*}
Evaluating the gradient at $F^t$ now gives
\begin{equation*}
\begin{aligned}
\nabla_{F}J(F^t, W^t) = A^t F^t - \frac{2}{m}\sum_{i=1}^m T_{i}^*Y_{i} = A^tF^t - A^tF^{t+1},
\end{aligned}
\end{equation*}
so that the norm
\begin{align*}
\|\nabla_{F}J(F^t, W^t)\|_{\mathcal{H}^d} = \|A^t(F^t-F^{t+1})\|_{\mathcal{H}^d} \leq B \|F^t - F^{t+1}\|_{\mathcal{H}^d}.
\end{align*}
\end{proof}

\begin{proof}[\hypertarget{prf: binomial_concentration}Proof of \Cref{loc:bounds_for_rmin_via_binomial_and_poisson_mgfs.statement_(chernoff_bound_for_binomial_or_poisson_random_variables)}]
First suppose $X \sim \mathrm{Bin}(R,p)$. By the Chernoff inequality for left tails (see \citet[Exercise 2.3.2]{vershynin_HighDimensionalProbability_2018}),
\begin{equation*}
\begin{aligned}
\mathbb{P}\{ X \leq \tau Rp \} &\leq e^{ -Rp }\left( \frac{e}{\tau } \right)^{\tau Rp} \\
&= e^{ -Rp }\left( e^{ 1 +\log (1/\tau) } \right)^{\tau Rp} \\
&= \exp\left( -Rp + \tau Rp + \tau Rp \log \frac{1}{\tau} \right),
\end{aligned}
\end{equation*}
as desired. 

Now let $X \sim \mathrm{Poi}(\lambda)$ for $\lambda=Rp$ and define $X_{N} \sim \mathrm{Bin}(N, \lambda /N)$. By the Poisson limit theorem,
\begin{equation*}
\begin{aligned}
\mathbb{P}\left\{ X \leq \tau\lambda \right\} &= \lim_{N \to \infty} \mathbb{P}\left\{ X_{N} \leq \tau\lambda \right\} \\
&=\lim_{N \to \infty} \mathbb{P}\left\{ X_{N} \leq \tau \cdot N (\lambda /N) \right\}  \\
&\leq \lim_{N \to \infty} \exp\left( -\lambda\left( 1-\tau-\tau \log \frac{1}{\tau} \right) \right) \\
&= \exp\left( -\lambda\left( 1-\tau-\tau \log \frac{1}{\tau} \right) \right),
\end{aligned}
\end{equation*}
where the second inequality uses the bound for Binomial random variables that we just obtained.
\end{proof}

\begin{proof}[\hypertarget{prf: bounded_iterates}Proof of \Cref{lem: bounded_iterates}]

Since $\mathcal C_\tau$ is a bounded set and $W^t \in \mathcal C_\tau \ \forall t$, it is clear that $(W^t)_{t \in \mathbb N}$ is bounded. For proving boundedness of $(\boldsymbol \theta^t)_{t \in \mathbb N}$, by \eqref{eq:theta_update} we observe that 
\begin{equation*}
    \begin{aligned}
        \| \boldsymbol{\theta}^{t+1} \|_F  &= \| \tfrac{1}{\sqrt{ m }}(\mathbf{K}^2+\lambda \mathbf{K}+\omega \mathbf{K}\overline{\mathbf{L}}_{W^t}\mathbf{K})^+ \mathbf{K}\mathbf{Y} \|_F \\
        & \leq  \tfrac{1}{\sqrt{m}} ||\mathbf{K}\|_{\text{op}} \| \mathbf Y \|_F \| \left( \mathbf{K}^2+\lambda \mathbf{K}+\omega \mathbf{K}\overline{\mathbf{L}}_{W^t}\mathbf{K} \right)^+\|_{\text{op}} \\
        &\leq  \tfrac{1}{\sqrt{m}}  ||\mathbf{K}\|_{\text{op}} \| \mathbf Y \|_F \| \left( \mathbf{K}^2+\lambda \mathbf{K} \right)^+\|_{\text{op}},
    \end{aligned}
\end{equation*}
where the last inequality uses the fact that $\mathbf{K}\overline{\mathbf{L}}_{W^t}\mathbf{K}$ is positive semidefinite together with properties of the pseudoinverse.
We further get that $(F^t)_{t \in \mathbb N}$ is bounded since for any $\boldsymbol{\theta}$, 
\begin{align*}
        \| F_{\boldsymbol{\theta}}\|_{\mathcal H^d}^2 & = \tr \left( \boldsymbol{\theta}^T \mathbf K \boldsymbol{\theta}\right) = \langle \boldsymbol{\theta}, \mathbf{K} \boldsymbol{\theta} \rangle_F \leq \| \boldsymbol{\theta} \|_F^2 \| \mathbf K \|_{\text{op}}.
\end{align*}
\end{proof}

\begin{proof}[\hypertarget{prf: limit_pts_critical}Proof of \Cref{lem: limit_pts_critical}]
We start by showing that
\begin{align}
\label{eq: square-summ}
    \sum_{t = 1}^\infty \| (\boldsymbol{\theta}^{t+1}, W^{t+1}) - (\boldsymbol{\theta}^{t}, W^{t})\|^2 < \infty.
\end{align}
To do this, we use the extended-valued function $\Jinf: \mathbb{R}^{m \times d} \times \mathbb{R}^{\mathcal{E}_\tau} \to \mathbb{R}\cup \{+\infty\}$, which coincides with $\mathbf{J}$ if $W \in \mathcal{C}_\tau$ and outputs $+\infty$ otherwise. 

By \Cref{loc:remark_(representer_theorem)}, the function $\Jinf(\cdot, W^t)$ is $\mu$-strongly convex on $\mathbf{K}\mathbb{R}^{m \times d}$. This guarantees that
\begin{equation}\label{eq: theta-strongly-convex-inequality}
\Jinf(\boldsymbol{\theta}^{t+1}, W^{t}) \geq \Jinf(\boldsymbol{\theta}^t, W^{t}) + \langle \nabla_{\theta} \Jinf(\boldsymbol{\theta}^{t+1},W^t), \boldsymbol{\theta}^{t+1} - \boldsymbol{\theta}^{t} \rangle + \frac{\mu}{2} \|\boldsymbol{\theta}^{t+1}- \boldsymbol{\theta}^{t} \|_F^2.
\end{equation}
By \eqref{eq: finite-dim-iterates} and \Cref{loc:remark_(representer_theorem)}, $\boldsymbol \theta^{t+1} \in \argmin_{\boldsymbol{\theta}   \in \mathbf K \mathbb R^{m \times d}} \Jinf ({\boldsymbol{\theta}}, W^t ) \subseteq \argmin_{\boldsymbol{\theta}   \in  \mathbb R^{m \times d}} \Jinf ({\boldsymbol{\theta}}, W^t )$ and
so $\nabla_{\theta} \Jinf(\boldsymbol{\theta}^{t+1},W^t) = 0$. Combining this with strong-convexity \eqref{eq: theta-strongly-convex-inequality}, we get
\begin{equation*}
    \frac{\mu}{2} \| \boldsymbol{\theta}^{t+1} - \boldsymbol{\theta}^t \|_F^2 \leq  \Jinf(\boldsymbol{\theta}^t, W^{t}) - \Jinf(\boldsymbol{\theta}^{t+1}, W^t).
\end{equation*}

Similarly, 
by \Cref{prop: W_min_solution}, for each $\boldsymbol \theta^t \in \mathbb R^{m \times d}$, the function $\Jinf(\boldsymbol \theta^t, \cdot )$ is $(\omega /2m)$-strongly convex. For any $Z \in \partial_W \Jinf(\boldsymbol \theta^{t+1}, W^{t+1})$
\begin{equation}\label{eq: W-strongly-convex-inequality}
\Jinf(\boldsymbol{\theta}^{t+1}, W^{t+1}) \geq \Jinf(\boldsymbol{\theta}^{t+1}, W^{t}) + \langle Z, W^{t+1} - W^t\rangle_F + \frac{\omega}{4m} \| W^{t+1} - W^t \|_F^2.
\end{equation}
Since $W^{t+1} \in \argmin_{W} \Jinf (\boldsymbol \theta^{t+1}, W )$, we have $0 \in \partial_W \Jinf(\boldsymbol \theta^{t+1}, W^{t+1})$. Now,
\begin{equation*}
    \frac{\omega}{4m} \| W^{t+1} - W^t \|_F^2 \leq \Jinf(\boldsymbol{\theta}^{t+1}, W^{t}) - \Jinf(\boldsymbol{\theta}^{t+1}, W^{t+1}).
\end{equation*}
Thus, there exists $C>0$ such that for all $t \in \mathbb N$,
\begin{equation}\label{eq: sufficient_decrease}
\| (\boldsymbol{\theta}^{t+1}, W^{t+1}) - (\boldsymbol{\theta}^{t}, W^{t})\|^2  \leq C ( \Jinf(\boldsymbol{\theta}^{t}, W^{t}) - \Jinf(\boldsymbol{\theta}^{t+1}, W^{t+1}) ).
\end{equation}
Summing up over all $t \in \mathbb N$, we get
\begin{align*}
    \sum_{t = 1}^\infty \| (\boldsymbol{\theta}^{t+1}, W^{t+1}) - (\boldsymbol{\theta}^{t}, W^{t})\|^2 & \leq C \sum_{t=1}^\infty \left[\Jinf(\boldsymbol{\theta}^{t}, W^{t}) - \Jinf(\boldsymbol{\theta}^{t+1}, W^{t+1})  \right]\\
    & = C \left[ \Jinf(\boldsymbol{\theta}^{1}, W^{1}) - \lim_{t \to \infty} \Jinf(\boldsymbol{\theta}^{t}, W^{t}) \right] < \infty,
\end{align*}
because the iterates $(\boldsymbol{\theta}^t, W^t)_{t \in \mathbb{N}}$ are feasible, and $\Jinf$ is bounded below.

Now, let $(\bar{\boldsymbol \theta}, \bar W)$ denote an arbitrary limit point of $(\boldsymbol{\theta}^t, W^t)_{t \in \mathbb{N}}$  and let $(\boldsymbol{\theta}^{t_s}, W^{t_s})_{s \in \mathbb N}$ be a subsequence converging to $(\bar{\boldsymbol{\theta}}, \bar W)$.
By \eqref{eq: square-summ}, we know that
$\boldsymbol{\theta}^{t_s +1} \to \bar{\boldsymbol{\theta}}$ and $W^{t_s + 1} \to \bar{W}$ as $s \to \infty$. Again by \eqref{eq: finite-dim-iterates} and \Cref{loc:remark_(representer_theorem)},
\begin{equation*}
\begin{aligned}
\mathbf J(\boldsymbol{\theta}^{t_s+1}, W^{t_s}) & \leq \mathbf J(\boldsymbol{\theta}, W^{t_s})  \ \forall \boldsymbol{\theta} \in \mathbb R^{m \times d} \\
\mathbf J(\boldsymbol{\theta}^{t_s+1}, W^{t_s+1}) & \leq \mathbf J(\boldsymbol{\theta}^{t_s+1}, W)  \ \forall W \in \mathcal C_\tau.
\end{aligned}
\end{equation*}
Taking the limit of both sides and using the fact that $\mathbf J$ is continuous, we have that
\begin{align*}
\mathbf J(\bar{\boldsymbol{\theta}}, \bar{W}) & \leq \mathbf J(\boldsymbol{\theta}, \bar W)   \ \forall \boldsymbol{\theta} \in \mathbb R^{m \times d} \\
\mathbf J(\bar{\boldsymbol{\theta}}, \bar{W}) & \leq \mathbf J(\bar{\boldsymbol{\theta}}, W)   \ \forall W \in \mathcal C_\tau.
\end{align*}
 
\end{proof}

\begin{proof}[\hypertarget{prf: iterates-bdd-away-simplex}Proof of \Cref{res: iterates-bdd-away-simplex}]

Recall the $W$-update $W^{t+1} = \text{RowNormalize}(\tilde{W}^{t+1})$, where
\begin{equation*}
\tilde{W}^t_{ik} = \exp \left( - \tfrac{1}{s_1^2} |\mathbf{F}^{t}_{k\cdot}-\mathbf{F}^{t}_{i\cdot}|^2 \right) \exp \left( - \tfrac{1}{s_2^2} \euc{q_k - q_i}^2\right)\indicator_{\{ |q_{k}-q_{i}| \leq \tau \}}.
\end{equation*}
Noting that $\mathbf{F}^t = \sqrt{ m }\mathbf{K}\boldsymbol{\theta}^t$ and using the uniform bound on $(\boldsymbol{\theta}^t)_{t \in \mathbb{N}}$ from the proof of \Cref{lem: bounded_iterates} gives $\euc{\mathbf{F}_{k\cdot}^t - \mathbf{F}_{i\cdot}^t} \leq \euc{\mathbf{F}_{i\cdot}^t} + \euc{\mathbf{F}_{k\cdot}^t} \leq 2 \fro{\mathbf{F}^t} \leq 2 \op{(\mathbf{K}^2 + \lambda \mathbf{K})^+}\op{\mathbf{K}}^2 \fro{\mathbf{Y}}$. Consequently, for $(i,k) \in \mathcal{E}_{\tau}$ (so that $\euc{q_{k}-q_{i}}\leq \tau$), we have
\begin{equation*}
\begin{aligned}
\tilde{W}^t_{ik} \geq \exp \left( - \tfrac{2}{s_{1}^2} \op{(\mathbf{K}^2 + \lambda \mathbf{K})^+}\op{\mathbf{K}}^2 \fro{\mathbf{Y}} \right) \exp \left( - \tfrac{\tau^2}{s_{2}^2} \right) =: \tilde{\nu} >0.
\end{aligned}
\end{equation*}
Finally, because $0 \leq \tilde{W} \leq 1$ (entry-wise), we conclude that
\begin{align*}
W^t_{ik} = \frac{\tilde{W}^t_{ik}}{\sum_{\ell}\tilde{W}^t_{i\ell}} \geq \frac{\tilde{\nu}}{m} =: \nu >0.
\end{align*}
\end{proof}

\section{Supplementary information on numerical experiments}
\label{appsec: exp-details}

\subsection{Hyperparameter settings for STARK}
\label{appsubsec: hyperparameters}

Our method STARK depends on several hyperparameters: the length scale $\lsc$ of the kernel function, the values $s_1$, $s_2$ and $\tau$ involved in the graph construction, the regularization strengths $\lambda$ and $\omega$, and the {number of iterations $N$} of the block coordinate descent scheme. 
We employ useful heuristics that automatically fix many of these hyperparameters so that minimal additional tweaking is required.
The choices involved are based on various numerical experiments and parameter sweeps.

The kernel length scale $\lsc>0$ is set to be a neighbourhood radius on the set of pixel locations such that a pixel has 7 neighbours on average that are at a distance at most $\lsc$ from it. 
Naturally, this length scale is also meaningful in constructing the weighted graph on the pixels. Set $s_{2}=\lsc$ and $\tau=\frac{3}{2}\lsc$. As \eqref{eq: W_update} suggests, $\tau$ sets a hard cutoff on the graph weights whereas $s_{2}$ defines a softer threshold, justifying $s_{2}<\tau$.

The parameter $s_{1}>0$ represents a length scale in gene expression space instead. A strategy for choosing it as follows.
\begin{itemize}
\item Compute the first iterate $F^1$ in Algorithm 3.1, which does not require $s_{1}$. This is already a decent estimate of the ground truth $F^\star$ (see the performance of the \textit{spatial} variant in \Cref{fig:spatial_oracle_comparisons}).
\item Set $s_{1}$ to be the 0.75-th quantile of the set $\{ \euc{\mathbf{F}_{i\cdot}^1 - \mathbf{F}^1_{k\cdot}} \mid (i,k) \in \mathcal{E}_{\tau},\, i \neq k \}$.
\end{itemize}

For regularization strengths $\lambda$ and $\omega$, it is convenient to set $\lambda=\alpha \lambda_{\mathrm{rel}}$, $\omega=\alpha \omega_{\mathrm{rel}}$, where $\alpha>0$ quantifies an `overall' regularization strength, and $\lambda_{\mathrm{rel}}, \omega_{\mathrm{rel}}>0$ determine the relative strengths of the ridge regularization and the graph Laplacian regularization terms in \eqref{eq: objective} respectively. We fix $\lambda_{\mathrm{rel}}=\op{\mathbf{K}}$ as a natural scale, and tune $\omega_{\mathrm{rel}}=6$.

The overall strength $\alpha$ should be set according to the noise level in the measurements; noisier data typically require more regularization. By \eqref{eq: expected norm noise}, we know that
\begin{equation*}
\mathbb{E} \left[ \mathrm{Fit}(F^\star) \mid R_{1},\dots,R_{m} \right]  :=\mathbb{E}\left[ \frac{1}{m} \sum_{i=1}^m \euc{Y_{i}-F^\star(q_{i})}^2 \,\Big|\, R_{1},\dots,R_{m} \right]  \simeq \frac{1}{m} \sum_{i=1}^m \min\left( \frac{1}{R_{i}}, 1 \right).
\end{equation*}
Therefore, it is reasonable to adjust the strength $\alpha$ such that our estimate $\bar{F}$ exhibits a comparable noise level i.e. $\mathrm{Fit}(\bar{F}) \simeq \mathbb{E}\,\mathrm{Fit}(F^\star)$. To avoid running the full iterative algorithm multiple times while adjusting $\alpha$ however, we actually use the first iterate $F^1$ as a proxy for $\bar{F}$. Briefly, we employ a bisection search to find a value for $\alpha$ such that
\begin{equation*}
\mathrm{Fit}(F^1) \simeq p \cdot \frac{1}{m} \sum_{i=1}^m \min\left( \frac{1}{R_{i}}, 1 \right),
\end{equation*}
where $p$ is a value close to 1. Tuning $p$ between 0.6 and 0.8 worked well in experiments.

Finally, the number of iterations $N$ is set to 7. A large number of iterations is not required here because the first iterate $F^1$, as discussed earlier, is already a reasonable estimate of the ground truth. Additionally, {we also observe fairly quick convergence of the iterates}.

\subsection{Details on computing label transfer accuracy and kNN overlap}
\label{appsubsec: metrics_details}

We need to following transformations on the original and denoised gene expression matrices $\mathbf{F}_0$ and $\bar{\mathbf{F}}$ in order to compute both metrics:
\begin{itemize}
    \item Transform the rows $(\mathbf{F}_{0})_{1\cdot},\dots,(\mathbf{F}_{0})_{m\cdot} \in \mathbb{R}^d$ of $\mathbf{F}_{0}$ elementwise by the function $t \mapsto \log(1+10^4\,t)$, center the resulting vectors by their mean, and perform principal component analysis (PCA) with $r$ components to yield eigenvectors $w_{1},\dots,w_{r} \in \mathbb{R}^d$, and principal component scores $Z_{1},\dots,Z_{m} \in \mathbb{R}^r$ for the pixels $i=1,\dots,m$.
    \item Transform the rows $\bar{\mathbf{F}}_{1\cdot},\dots,\bar{\mathbf{F}}_{m\cdot} \in \mathbb{R}^d$ of $\bar{\mathbf{F}}$ elementwise by the function $t \mapsto \log(1+10^4\,t)$, center the resulting vectors by their mean, and compute their corresponding scores $\bar{Z}_{1},\dots,\bar{Z}_{m} \in \mathbb{R}^r$ with respect to the eigenvectors $w_{1},\dots,w_{r}$ from the previous step. Keeping the eigenvectors the same across these two steps allows for meaningfully computing distances between $\{ Z_{i} \}_{i=1}^m$ and $\{ \bar{Z}_{i} \}_{i=1}^m$. 
\end{itemize}
We set $r=30$ in our experiments.

\paragraph{Computing the label transfer accuracy} 
The precise steps are as follows.
\begin{enumerate}
\item At each pixel $i$, the original score $Z_{i}$ is paired with a cell type label $l_{i}$. Train a $k$-nearest neighbour classifier on the pairs $\{ (Z_{i}, l_{i})\}_{i=1}^m$, using the Euclidean metric. 
\item Test this classifier on the scores $\bar{Z}_{1},\dots,\bar{Z}_{m}$ resulting from the denoised matrix to generate predicted labels $\bar{l}_{1},\dots,\bar{l}_{m}$. 
\item Compute the label transfer accuracy as the proportion of correctly predicted labels: 
\begin{align*}
\frac{1}{m} |\{ i \in [m] \mid l_{i}=\bar{l}_{i} \}|.
\end{align*}
\end{enumerate}
We use $k=9$ for the $k$-nearest neighbour classifier.

\paragraph{Computing the kNN graphs}
The precise steps are as follows.
\begin{enumerate}
\item Construct a directed graph on the points $Z_{1},\dots,Z_{m}$ by connecting each $Z_{i}$ to $k$ other $Z_{i'}$ that are closest to it in Euclidean distance $\euc{Z_{i}-Z_{i'}}$. Define the adjacency matrix $A_0 \in \{ 0,1 \}^{m\times m}$ such that $(A_0)_{ii'}=1$ if $Z_{i'}$ is one of the $k$ closest vectors to $Z_{i}$, and $(A_0)_{ii'}=0$ otherwise.
\item Repeat this construction on $\bar{Z}_{1},\dots, \bar{Z}_{m}$, and compute the corresponding adjacency matrix $\bar{A}$. 
\item Compute the overlap 
\begin{align*}
    \frac{1}{m} \langle A_0, \bar{A}\rangle_F.
\end{align*}
\end{enumerate}
We set $k=50$ for this metric, which is reasonable following \citet{ahlmann-eltze_ComparisonTransformationsSinglecell_2023}.

\subsection{Hyperparameters and transformations for SPROD, GraphPCA and STAGATE}
\label{appsubsec: hyperparam-competing}

We sweep over the following hyperparameters to assess and optimize the performance of SPROD, GraphPCA and STAGATE.
\begin{itemize}
    \item \textbf{SPROD}. The spot neighborhood radius ratio \texttt{sprod\_R}, the dimension of the latent space \texttt{sprod\_latent\_dim}, the regularization term for the similarity graph construction \texttt{sprod\_graph\_reg}, and the regularization term for the denoising weights \texttt{sprod\_weight\_reg}; see \url{https://github.com/yunguan-wang/SPROD} for further details. 
    \item \textbf{GraphPCA}. The number of principal components \texttt{n\_components}, the number of nearest neighbours per pixel \texttt{n\_neighbors}, and the regularization strength \texttt{\_lambda}; see \url{https://github.com/YANG-ERA/GraphPCA}.
    \item \textbf{STAGATE}. The cutoff radius for the neighbourhood graph \texttt{rad\_cutoff}; see \url{https://stagate.readthedocs.io/en/latest/T6_Denoising.html}. Another tuneable hyperparameter is \texttt{alpha}, which controls the strength of the cell-type-aware module. However, the authors \citet{dong_DecipheringSpatialDomains_2022} do not recommend a non-zero value of \texttt{alpha} for high spatial resolution datasets such as MOSTA.
\end{itemize}

\subsubsection*{Hyperparameter tuning process and final choices at test time}

For each downsampling level $R$ tested, we generate a noisy matrix $\mathbf{Y}_{\text{val}}(R)$ from the original counts $\mathbf{C}_0$ (as in \Cref{subsec: denoising_interpolation_tests}), and this is treated as validation data. $\mathbf{Y}_{\text{val}}(R)$ is fed into the denoising methods being tuned, and the performance metrics of the denoised outputs are tracked over various hyperparameter combinations. The best-performing hyperparameters with respect to label transfer accuracy are then selected for this downsampling level (one could also similarly optimize for kNN overlap or relative error). These hyperparameter choices, reported in \Cref{tab: sprod-hyperparams}, \Cref{tab: GPCA-hyperparams} and \Cref{tab: STAGATE-hyperparams}, are used for the final comparison tests presented in \Cref{fig:denoising_test_comparisons} and \Cref{fig:denoising_comparison_images}.
Note that in these final tests, the noisy matrices $\mathbf{Y}$ are generated by downsampling $\mathbf{C}_0$ \emph{independently} from the validation data $\mathbf{Y}_{\text{val}}(R)$. 

\begin{table}[h]
    \centering
    \begin{filecontents*}{Tables/params-250720-vanilla-nonorm-sprod_clopt.csv}
,downsampling_ratio,Reads per pixel,SPROD_R_val,SPROD_latent_dim,SPROD_graph_reg,reg_lam_sprod
0,0.001,13.8,0.025,20,5,3.0
1,0.002,27.6,0.015,10,5,6.0
2,0.004,55.2,0.015,20,10,3.0
3,0.007,96.6,0.015,10,1,1.0
4,0.01,138.0,0.015,5,1,1.0
5,0.015,206.9,0.015,5,1,0.625
6,0.02,275.9,0.015,5,1,0.625
\end{filecontents*}
\csvreader[tabular=c|c c c c, table head= Reads per pixel & \texttt{sprod\_R} & \texttt{sprod\_latent\_dim} & \texttt{sprod\_graph\_reg} & \texttt{sprod\_weight\_reg}\\\hline, late after last line=\\\hline]{Tables/params-250720-vanilla-nonorm-sprod_clopt.csv}{}{%
        \csvcoliii & \csvcoliv & \csvcolv & \csvcolvi & \csvcolvii %
    }
    \caption{\textbf{Hyperparameters for SPROD}. The first column quantifies the downsampling level in terms of reads per pixel. The reported hyperparameter combinations are the ones that produced the highest label transfer accuracies on the validation data $\mathbf{Y}_{\mathrm{val}}(R)$ at the various downsampling levels.} 
    \label{tab: sprod-hyperparams}
\end{table}

\begin{table}[h]
    \centering
    \begin{filecontents*}{Tables/params-250729-GraphPCA_clopt.csv}
,downsampling_ratio,Reads per pixel,_lambda,n_neighbors,n_components
0,0.001,13.8,0.8,20,80
1,0.002,27.6,0.8,20,80
2,0.004,55.2,0.7,15,80
3,0.007,96.6,0.8,8,80
4,0.01,138.0,0.8,10,80
5,0.015,206.9,0.8,8,80
6,0.02,275.9,0.8,6,80
\end{filecontents*}
\csvreader[tabular=c|c c c, table head= Reads per pixel & \texttt{\_lambda} & \texttt{n\_neighbors} & \texttt{n\_components}\\\hline, late after last line=\\\hline]{Tables/params-250729-GraphPCA_clopt.csv}{}{%
        \csvcoliii & \csvcoliv & \csvcolv & \csvcolvi %
    }
    \caption{\textbf{Hyperparameters for GraphPCA}. The first column quantifies the downsampling level in terms of reads per pixel. The reported hyperparameter combinations are the ones that produced the highest label transfer accuracies on the validation data $\mathbf{Y}_{\mathrm{val}}(R)$ at the various downsampling levels.} 
    \label{tab: GPCA-hyperparams}
\end{table}

\begin{table}[h]
    \centering
    \begin{filecontents*}{Tables/params-250729-STAGATE_clopt.csv}
,downsampling_ratio,Reads per pixel,rad_cutoff
0,0.001,13.8,3.3
1,0.002,27.6,3.3
2,0.004,55.2,3.3
3,0.007,96.6,3.0
4,0.01,138.0,3.0
5,0.015,206.9,3.3
6,0.02,275.9,3.3
\end{filecontents*}
\csvreader[tabular=c|c, table head= Reads per pixel & \texttt{rad\_cutoff}\\\hline, late after last line=\\\hline]{Tables/params-250729-STAGATE_clopt.csv}{}{%
        \csvcoliii & \csvcoliv %
    }
    \caption{\textbf{Hyperparameters for STAGATE}. The left column quantifies the downsampling level in terms of reads per pixel. The reported hyperparameters are the ones that produced the highest label transfer accuracies on the validation data $\mathbf{Y}_{\mathrm{val}}(R)$ at the various downsampling levels.} 
    \label{tab: STAGATE-hyperparams}
\end{table}

\subsubsection*{Transformations on the data}

In the current data pipeline of denoising followed by evaluation, data transformation by the function $\texttt{log1p}$ appears as a step after denoising and before computing label transfer accuracy and kNN overlap (henceforth referred to as ``transform after denoising"). 
As discussed by \citet{ahlmann-eltze_ComparisonTransformationsSinglecell_2023}, such transformations are often beneficial for downstream analyses. 
The authors of SPROD, GraphPCA and STAGATE in fact suggest applying $\texttt{log1p}$ or analytic Pearson residuals~\cite{lause_AnalyticPearsonResiduals_2021} even before denoising (``transform before denoising"). In our experiments however, we
observed significantly worse label transfer accuracies
if these transformations are done too early;
see the following paragraph for details.
Consequently, for the comparison tests presented in \Cref{fig:denoising_test_comparisons} and \Cref{fig:denoising_comparison_images}, we provide (scaled) counts matrices as inputs to these methods, without any nonlinear transformations.

For a fixed downsampling level of $\sim 97$ reads per pixel, we generate a noisy matrix $\mathbf{Y}$. SPROD, GraphPCA and STAGATE are run with various hyperpameter combinations under both scenarios ``transform before denoising" and ``transform after denoising". The resulting label transfer accuracies are presented in \Cref{fig:sprod-pre-post} for SPROD, \Cref{fig:gpca-pre-post} for GraphPCA, and \Cref{fig:stagate-pre-post} for STAGATE. Note that the ``transform before denoising" scenario for GraphPCA employs analytic Pearson residuals as recommended by \citet{yang_GraphPCAFastInterpretable_2024}. The remaining cases simply use the standard \texttt{log1p} transformation. 

\begin{figure}
    \centering
    \includegraphics[width=\textwidth]{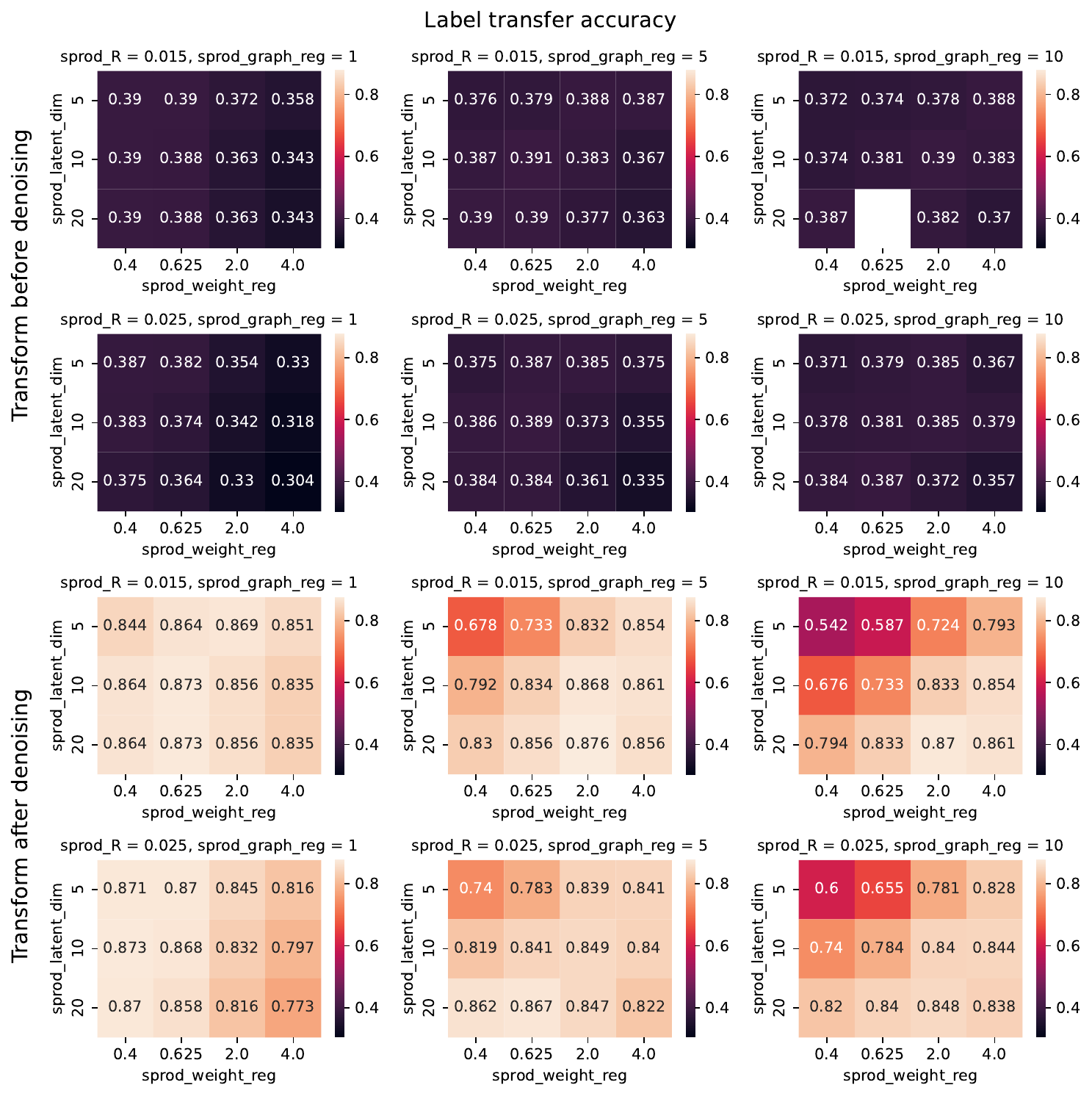}
    \caption{Comparing the denoising performance of SPROD under the scenarios ``transform before denoising" and ``transform after denoising", in terms of label transfer accuracy, over various hyperparameter combinations. Transforming after denoising achieves much higher label transfer accuracies.}
    \label{fig:sprod-pre-post}
\end{figure}

\begin{figure}
    \centering
    \includegraphics[width=\textwidth]{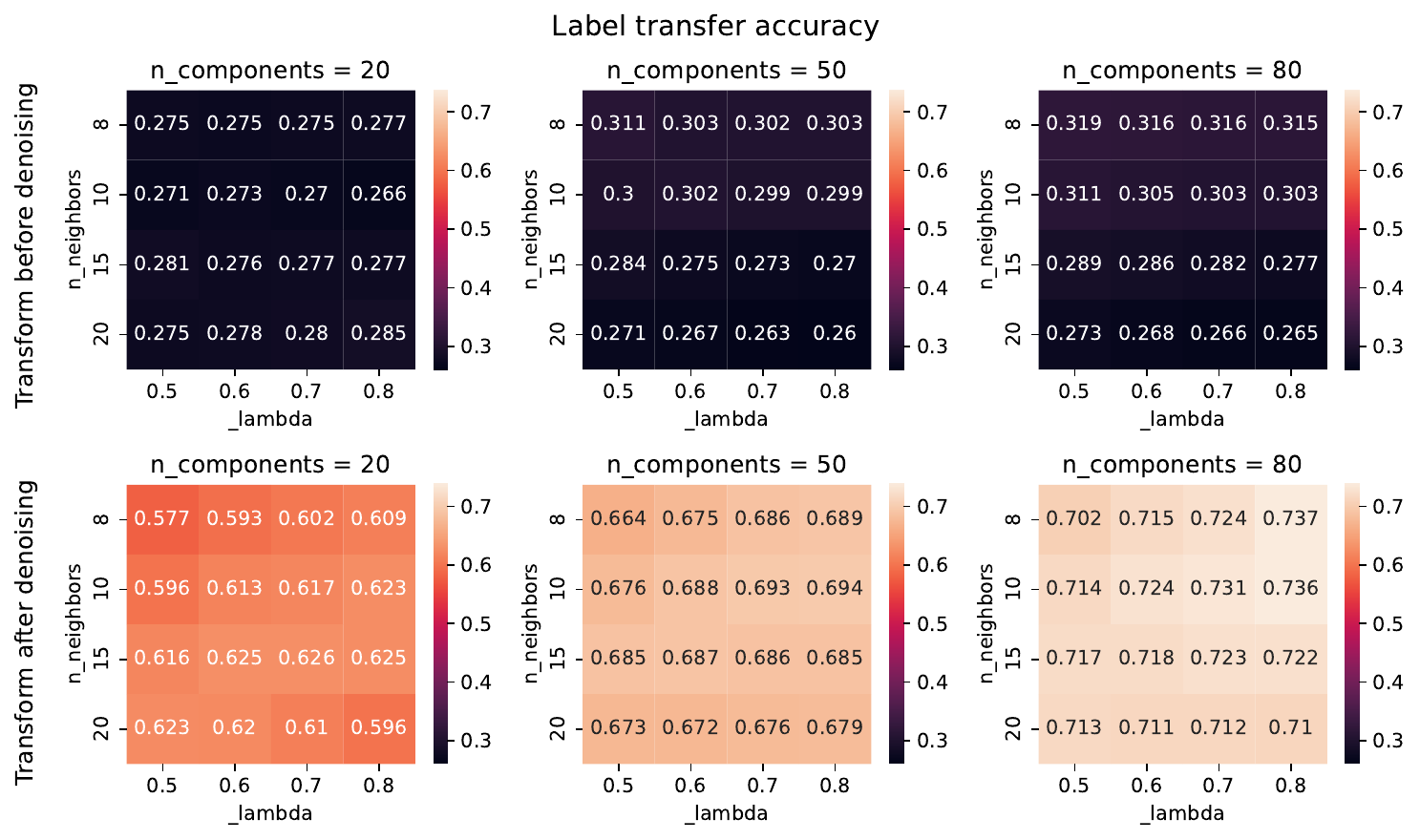}
    \caption{Comparing the denoising performance of GraphPCA under the scenarios ``transform before denoising" and ``transform after denoising", in terms of label transfer accuracy, over various hyperparameter combinations. Transforming after denoising achieves much higher label transfer accuracies.}
    \label{fig:gpca-pre-post}
\end{figure}

\begin{figure}
    \centering
    \includegraphics[width=0.4\textwidth]{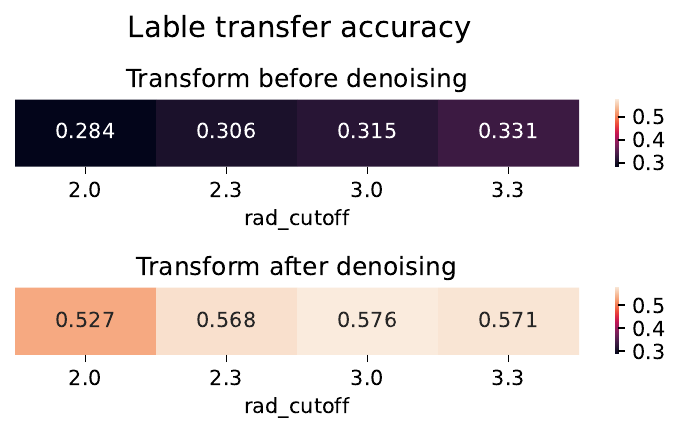}
    \caption{Comparing the denoising performance of STAGATE under the scenarios ``transform before denoising" and ``transform after denoising", in terms of label transfer accuracy, over different values of \texttt{rad\_cutoff}. Transforming after denoising achieves higher label transfer accuracies.}
    \label{fig:stagate-pre-post}
\end{figure}

\end{document}